%% file: main.tex
\documentclass[11pt]{article}

\ifdefined\XeTeXrevision\else\ifdefined\pdfoutput\pdfoutput=1\fi\fi

\usepackage[letterpaper,margin=1in,headheight=14pt]{geometry}

\usepackage[charter]{mathdesign}
\usepackage[T1]{fontenc}
\usepackage{iftex}
\ifPDFTeX\usepackage[utf8]{inputenc}\fi
\usepackage{microtype}
\usepackage[table]{xcolor}
\definecolor{maintext}{HTML}{1F2430}
\definecolor{accent}{HTML}{0E8F55}
\definecolor{linkaccent}{HTML}{0B7D4A}
\arrayrulecolor{maintext}
\AtBeginDocument{\color{maintext}}

\usepackage{graphicx}
\usepackage{booktabs}
\usepackage{array}
\usepackage{amsmath}
\usepackage{amsthm}
\graphicspath{{figures/}}
\usepackage[font=small,labelfont=bf,skip=8pt,width=0.92\textwidth]{caption}

\numberwithin{equation}{section}
\newtheorem{lemma}{Lemma}[section]
\newtheorem{corollary}[lemma]{Corollary}
\theoremstyle{definition}
\newtheorem{remark}[lemma]{Remark}

\usepackage{titlesec}
\titleformat{\section}{\Large\bfseries\color{maintext}}{\textcolor{accent}{\thesection.}}{0.5em}{}
\titleformat{\subsection}{\large\bfseries\color{maintext}}{\textcolor{accent}{\thesubsection}}{0.5em}{}
\titlespacing*{\section}{0pt}{1.6em}{0.6em}

\usepackage{fancyhdr}
\usepackage[round,authoryear]{natbib}
\usepackage[colorlinks=true,
            linkcolor=linkaccent,
            citecolor=linkaccent,
            urlcolor=linkaccent]{hyperref}
\hypersetup{pdftitle={Post-Training Science for Supervised Fine-Tuning},
            pdfauthor={Charles O'Neill, Mudith Jayasekara, Harry Partridge}}

\newcommand{\affil}[1]{\textsuperscript{#1}}

\begin{document}
\thispagestyle{empty}

\noindent\includegraphics[height=20pt]{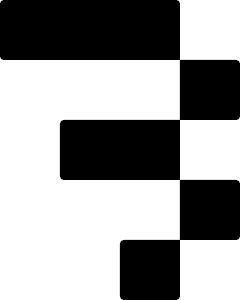}\par
\vspace{1.5em}
{\raggedright\fontsize{23}{27}\selectfont\bfseries
Post-Training Science for Supervised Fine-Tuning\par}

\vspace{1.1em}

{\raggedright\large
\textbf{Charles O'Neill}\affil{1} \quad \textbf{Mudith Jayasekara}\affil{1} \quad \textbf{Harry Partridge}\affil{1}\par}
\vspace{0.25em}
{\raggedright\affil{1}Baseten\par}

\vspace{1.8em}

{\bfseries
Every supervised fine-tuning run forces the same chain of decisions, such as learning rate, batch size, LoRA or full fine-tuning, how many epochs, which optimiser, and what data to feed the model. Each of these is typically rediscovered from scratch for every new model and dataset. Here we measure them under one instrument: a sweep that varies one lever at a time, and spans dense and mixture-of-experts models in two families (Qwen3 and Llama), on four real-world customer SFT datasets, for both LoRA and full fine-tuning. These datasets give a controlled testbed: each task carries an evaluation built with the customer, and its training data is produced by iterative supervised fine-tuning that refines model outputs until they pass that evaluation, so the supervised target is internally consistent and the task judge we report against is the criterion the data was built to satisfy. We ask how the optimal learning rate and batch size move with model scale, family, and data, and whether one selection rule transfers across them; what LoRA trades against full fine-tuning, and how its rank and $\alpha$ set what the adapter can learn; whether validation loss (or other metrics, such as loss landscape flatness) faithfully ranks downstream quality; whether post-training gains scale with model size and data volume, on a model ladder extended through mixtures-of-experts to 235B parameters; how many epochs to train before general instruction-following erodes; and whether a geometry-aware optimiser improves on AdamW. Each recommendation is paired with a measure of its uncertainty.\par}

\vspace{1.6em}

\input{sections/introduction}
\input{sections/qwen_batch_lr}
\input{sections/lora_hparams}
\input{sections/downstream_performance}
\input{sections/scaling_laws}
\input{sections/optimizers}
\input{sections/epochs}
\input{sections/conclusion}

\bibliographystyle{plainnat}
\bibliography{references}

\clearpage

{\small
 \setcounter{tocdepth}{1}
 \hypersetup{linkcolor=maintext}
 \tableofcontents}

\appendix
\titleformat{\section}{\Large\bfseries\color{maintext}}{\textcolor{accent}{Appendix \thesection.}}{0.5em}{}
\input{sections/appendix_timing}
\input{sections/appendix_frontiers}
\input{sections/appendix_hessian}
\input{sections/appendix_optimizers}

\end{document}

%% file: sections/introduction.tex
\section{Introduction}

Adapting an open base model to a task is now the dominant way language models reach production, yet the decisions that govern each supervised fine-tuning run are still made heuristically, as defaults inherited from pretraining, rules tuned on someone else's model, or a fresh ad-hoc sweep per job. Here we aim to replace those ad hoc decisions with controlled measurement on real customer tasks.

We answer one question per section. We sweep dense base models across two families (Qwen3 and Llama 3.1/3.2), plus several mixture-of-experts models as an out-of-architecture test, on four anonymised customer SFT datasets, for both LoRA and full fine-tuning, holding data, splits, and seeds fixed so that only the variable under study changes. A run enters the analysis only if it beat its pretrained baseline on validation NLL under the same masking, and cross-dataset comparisons are reported as improvement over that baseline, which removes the token-load confound of absolute loss (\S\ref{subsec:data}).

The four datasets give a controlled testbed. Each comes from a production task for which we built a reliable evaluation with the customer, and its training examples are generated by iterative supervised fine-tuning \citep{oneill2025isft}: a model drafts an output, an evaluator grades it and explains the failure, and the draft is refined until it passes, so every training pair already scores well under the task's own evaluation. This gives the sweep two properties it would not have on generic data. First, the supervised target is internally consistent, free of contradictory gold answers or label noise, so a change in validation NLL reflects the lever under study and not the data. Second, the criterion the data was built to satisfy is the production judge we score against downstream (\S\ref{sec:downstream}), which the customer validated. Appendix~\ref{app:isft} provides more details.

Section~\ref{sec:foundation} calibrates the two levers every run must set, learning rate and batch size. Across 0.6--32B and both families the optimal LoRA learning rate is flat at $10^{-3}$, roughly $33\times$ the full fine-tuning optimum, and global batch trades loss against cost and has no single best value. The rule transfers unchanged to a held-out 30B mixture-of-experts, which fine-tunes like a median $8.5$B dense model, and holds the monotone size trend when applied to 80B and 235B. Appendix~\ref{app:timing} measures the wall-clock, compute, and memory costs behind the batch frontier.

Section~\ref{sec:lorahp} varies the adapter the calibration held fixed. Rank adds usable capacity up to about $64$ and then plateaus, and $\alpha=32$ is the best scale in every cell, so the $r=64,\alpha=32$ default holds; rank $32$ is the cheaper setting, giving up at most $0.003$ nats.

Section~\ref{sec:downstream} asks when validation loss, the metric every selection here relies on, ranks downstream quality. Within a fixed (model, dataset, recipe) cell it does (within-dataset-standardised Spearman $-0.38$ to $-0.88$), so selection by loss is sound inside a recipe, but it does not transfer across model families, and the flatness of the converged minimum, measured by the Fisher trace, adds nothing reliable at matched loss.

Section~\ref{sec:scaling} asks which post-training resources scale predictably: model size, data counted in examples or in tokens, and the adapter's trainable-parameter budget, and whether LoRA and full fine-tuning scale differently. Its model ladder extends through Qwen3 mixtures-of-experts to 235B total parameters. On the dense ladders the at-defaults loss follows a saturating power law in size, steeper for full fine-tuning than for LoRA, and all three MoEs enter the trend at roughly the geometric mean of their active and total parameters. On Qwen3-4B, 8B, and 14B, tripling fresh examples lowers one-epoch LoRA validation loss in every experiment; the return is mostly dataset-specific rather than size-specific.

The remaining sections measure other components of post-training. Muon carries a narrow version of its pretraining advantage into SFT (\S\ref{sec:optimizers}): at a learning rate about $3\times$ below AdamW's it reaches a marginally lower validation NLL and a flatter minimum on all four datasets, with judged task quality even and more general instruction-following retained on the IFEval dataset. Our epoch study (\S\ref{sec:epochs}) finds validation loss overfits past about two epochs while judged quality does not improve and instruction-following erodes, and it separates repeated passes from fresh data: adding examples lowers loss but does not move the optimal epoch ceiling beyond approximately two.

We make three decisions throughout these experiments. Selections are made on validation NLL, an assumption
\S\ref{sec:downstream} tests. Cross-dataset comparisons use improvement over the
pretrained baseline rather than absolute loss. And with four datasets and at most six sizes per family, we quote magnitudes with uncertainty and treat directions as the reliable findings, with point estimates carrying explicit uncertainty.

%% file: sections/qwen_batch_lr.tex
\section{Learning Rate and Batch Size}
\label{sec:foundation}

Learning rate and batch size are the two decisions that most directly govern the cost and outcome of a supervised fine-tuning run, and the two most often rediscovered from scratch for each new model and dataset. The literature on both is almost entirely about pretraining, so it does not say how the optima move across model scale, family, adapter choice, and data.

For pretraining, the link between learning rate and batch size is established: the linear scaling rule ties learning rate to batch size \citep{goyal2017minibatch}; the gradient noise scale predicts a critical batch size beyond which extra parallelism yields no further step-efficiency \citep{mccandlish2018largebatch}; large sweeps show the optimal learning rate first tracks the batch and then saturates \citep{shallue2019dataparallelism, smith2018increasebatch}; and under maximal-update parametrisation the optimum scales roughly as inverse width \citep{yang2021mup}. For fine-tuning, \citet{thinkingmachines2025lora} report that LoRA's optimal learning rate sits about an order of magnitude above full fine-tuning's and that, tuned correctly, LoRA matches full fine-tuning, at a single scale. Whether these optima hold across model sizes, families, and datasets is untested. We therefore sweep learning rate against global batch size for both adapters across nine models, two families, and four datasets, fit a transferable selection rule, and test it on held-out families and datasets.

\subsection{Experimental setup}

We sweep learning rate against global batch size for nine base models in two families (Qwen3 at 0.6B, 1.7B, 4B, 8B, 14B, and 32B; Llama at 3.2-1B, 3.2-3B, and 3.1-8B)\footnote{We use Qwen3 rather than a more recent Qwen release because the study needs one family spanning a wide size range: the small dense models ($0.6$--$1.7$B) anchor the scaling axis, and the large models (up to $32$B, plus the mixtures-of-experts) must share that family for the size comparison to stay within-family. Qwen3 covers that ladder; the newer releases do not offer the same small-to-large dense range. We pick Llama 3.1/3.2 as the second family because its $1$B, $3$B, and $8$B sizes sit close to Qwen3's, so the cross-family check is at matched scale.} on four anonymised
customer SFT datasets: \textsf{security}, \textsf{leasing}, \textsf{support}, and \textsf{docs}. Every cell trains $5{,}000$ examples for one epoch under fixed seeds, splits, and an
assistant-only label policy. All nine models share identical grids: LoRA is the primary arm (rank 64, $\alpha=32$, all-linear, AdamW) over global batch $\{8,16,32,64\}$ and learning rate $\{3\!\times\!10^{-5},\dots,3\!\times\!10^{-3}\}$; full fine-tuning (FullFT) is a smaller reference arm over batch $\{16,64\}$ and learning rate
$\{3\!\times\!10^{-6},\dots,10^{-4}\}$. We use a cosine decay schedule.

An experiment (what we refer to as a ``cell'') is \emph{eligible} only if it reached a validation NLL strictly below the pretrained baseline under the same masking. $970$ of the $1{,}008$ planned cells qualify, and every excluded cell is a LoRA run at the top learning rate $3\!\times\!10^{-3}$ (\S\ref{subsec:instability}); full coverage and token statistics are in Appendix~\ref{app:data}. We select learning rates by validation NLL and report cross-dataset results as improvement over the pretrained baseline, since absolute NLL is token-confounded (\S\ref{subsec:data}). The ``basin'' of a cell is the set of learning rates within $0.01$ nats of its optimum.

\subsection{The optimal learning-rate law}
\label{subsec:law}

Following \citet{thinkingmachines2025lora}, we fit the optimal learning rate as a power law in model hidden size with a multiplier for LoRA over FullFT,
\begin{equation}
  \mathrm{LR} \;=\; C \cdot M_{\text{tuner}} \cdot
  \left(\frac{2000}{\text{hidden size}}\right)^{\,p + q_{\text{tuner}}},
  \qquad M_{\text{FullFT}}=1,\; q_{\text{FullFT}}=0,
\end{equation}
independently in each family (the reference width $2000$ is arbitrary and only sets the units of $C$). Rather than regress on noisy per-cell argmins, we score a candidate law by the validation loss its predicted learning rates would incur, interpolating each cell's loss-vs-learning-rate curve.

Table~\ref{tab:law} gives the fitted laws. The most interesting finding was the LoRA exponent: $p+q$ is
statistically indistinguishable from zero in both families, so the optimal LoRA learning rate is
\emph{flat at $10^{-3}$ across the entire $0.6$--$32$B range and both families}
(Figure~\ref{fig:law}). The FullFT optimum sits near $3\!\times\!10^{-5}$, an order of magnitude
lower. The remaining parameters are loosely identified: $C$ and the multiplier $M$ trade off along
a flat valley, and the FullFT exponent is a gentle downward drift in Qwen that the data cannot
distinguish from flat.

The selection rule also transfers. A law fit on one family predicts the other family's optimal learning
rates to within $0.004$ nats of its own law (Figure~\ref{fig:law-transfer}), and the
leave-one-dataset-out prediction error is median $0.0$ and 90th-percentile $\approx 0.5$ in
$\log_{10}$ learning rate. The loosely identified multiplier is the wrong quantity to transfer (the
per-cell discrete-best LoRA/FullFT ratio is $\approx\!33\times$ in both families) but the
predicted learning rate transfers. Committing to a single LoRA learning rate of $10^{-3}$
everywhere costs under $0.01$ nats relative to per-cell tuning.

\begin{table}[htbp]
  \footnotesize
  \centering
  \caption{\textbf{The optimal LoRA learning rate is flat in model size.} Fitted
  optimal-learning-rate law per family (loss-weighted fit; 95\% bootstrap confidence intervals in
  brackets, $B=250$ resamples). The LoRA exponent $p{+}q$ brackets zero in both families; $C$ (in
  units of $10^{-5}$) and the LoRA multiplier $M$ trade off in a wide valley and are loosely
  identified, as is the FullFT exponent $p$.}
  \label{tab:law}
  \begin{tabular}{lcccc}
    \toprule
    Family & $C\;(\times 10^{-5})$ & $M_{\text{LoRA}}$ & $p$ (FullFT) & $p{+}q$ (LoRA) \\
    \midrule
    Qwen3 & $3.9$ \;{\footnotesize$[3.0,6.2]$} & $26$ \;{\footnotesize$[16,33]$} & $+0.27$ \;{\footnotesize$[0.00,0.75]$} & $0.00$ \;{\footnotesize$[-0.00,0.03]$} \\
    Llama & $3.0$ \;{\footnotesize$[3.0,3.5]$} & $33$ \;{\footnotesize$[28,33]$} & $+0.00$ \;{\footnotesize$[0.00,0.97]$} & $0.00$ \;{\footnotesize$[-0.02,0.01]$} \\
    \bottomrule
  \end{tabular}
\end{table}

\begin{figure}[htbp]
  \centering
  \includegraphics[width=\linewidth]{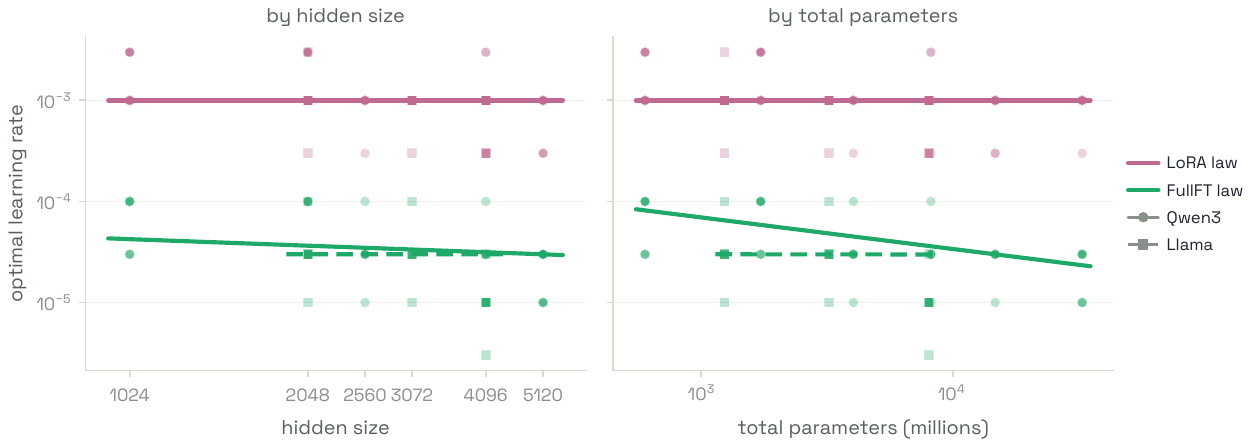}
  \caption{\textbf{The optimal LoRA learning rate is flat at $10^{-3}$ across $0.6$--$32$B in both
  families, an order of magnitude above FullFT.} Per-cell discrete optima (Qwen $\bullet$, Llama
  $\blacksquare$) with the loss-weighted law per tuner and family (solid Qwen, dashed Llama), on
  the hidden-size axis (left) and the total-parameter axis (right). Qwen3-14B and 32B share hidden
  size $5120$, which the parameter axis resolves.}
  \label{fig:law}
\end{figure}

\begin{figure}[htbp]
  \centering
  \includegraphics[width=\linewidth]{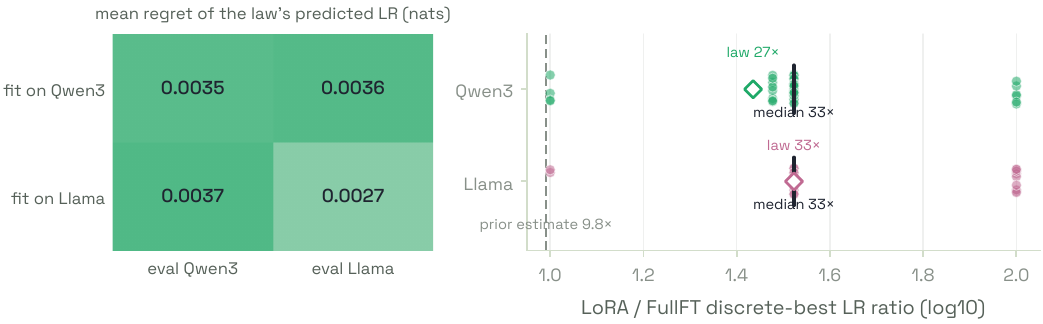}
  \caption{\textbf{The selection rule transfers across families.} Left: the cross-family transfer
  matrix; a law fit on one family incurs almost the same regret evaluated on the other family as
  on its own. Right: one point per matched cell; the median discrete-best LoRA/FullFT ratio is
  $\approx 33\times$ in both families, while the loss-weighted multiplier each law settles on (open
  diamonds) differs because the valley is flat.}
  \label{fig:law-transfer}
\end{figure}

\subsection{LoRA versus full fine-tuning}

Matched on model, dataset, and batch at the shared sizes $\{16,64\}$, FullFT is at or below LoRA on validation NLL in all $72$ comparisons. However, LoRA recovers a median
$98\%$ of FullFT's improvement over baseline (Table~\ref{tab:lora}) while training
$3.1$--$12.6\%$ of the parameters.

Global batch is a decision around both loss and cost. At a fixed one-epoch token budget a smaller batch takes more optimiser steps and reaches lower loss, so ranking by absolute NLL would trivially favour the smallest batch. Since one-epoch wall-clock is nearly batch-invariant and memory is what caps the batch (Appendix~\ref{app:timing}), the cost of a small batch is its optimiser steps. The optimal learning rate is nearly batch-invariant, drifting up only weakly with batch in Qwen ($+0.07$ dex per doubling for the smaller models) and not at all in Llama.

\begin{figure}[htbp]
  \centering
  \includegraphics[width=\linewidth]{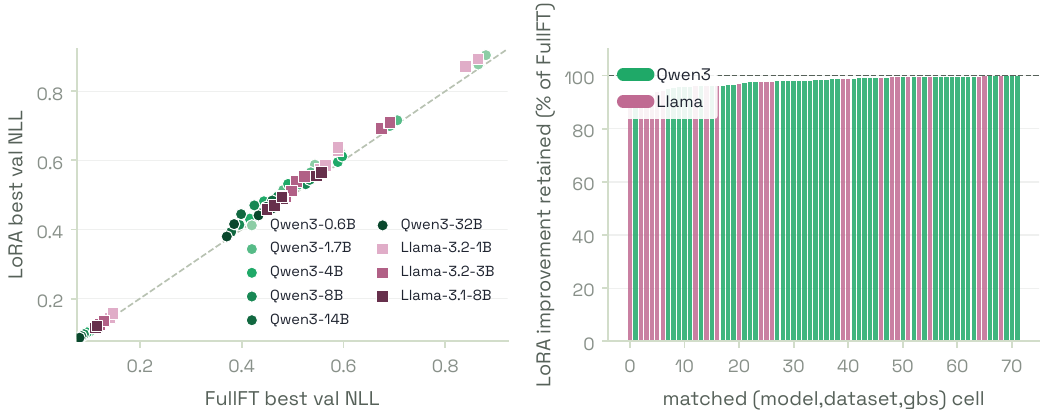}
  \caption{\textbf{FullFT wins all $72$ matched comparisons, but the margins are small.} Left:
  paired best validation NLL at matched settings across all nine models (Qwen $\bullet$, Llama
  $\blacksquare$). Right: the fraction of FullFT's improvement over baseline that LoRA recovers,
  coloured by family.}
  \label{fig:loravsfull}
\end{figure}

\subsection{Dataset composition determines the achievable loss}
\label{subsec:data}

The design fixes examples and epochs but not tokens: the four datasets differ by $18$--$23\times$ in assistant (loss) tokens per example (Appendix~\ref{app:data}). We found that total tokens sets the achievable loss. Within each family the pretrained baseline NLL explains almost none of the final NLL ($R^2\!=\!0.01$ Qwen, $0.02$ Llama) while token load explains most of it ($R^2\!=\!0.58$ and $0.61$): more tokens give worse absolute NLL and smaller improvement over baseline, with identical ordering and sign in both families (Figure~\ref{fig:tokens-floor}). Absolute NLL is therefore not comparable across datasets; we compare improvement over baseline.

A variance decomposition shows the same thing (Figure~\ref{fig:vardecomp}). Dataset identity (here equivalent to token composition) explains $56$--$87\%$ of the variance in final NLL and $69$--$88\%$ in improvement, while global batch and learning rate together explain almost none (each $\le 0.07$).

However, they do change the optimum, albeit slightly. Dataset explains $16$--$19\%$ of the variance in the discrete-best LoRA learning rate ($27$--$33\%$ for FullFT), and the optimum drifts up with token count, $+0.18$ dex per decade for LoRA and $+0.41$ for FullFT after removing model and batch effects, and the drift survives restricting to non-railed cells
(Figure~\ref{fig:tok-optlr}). Over our datasets' $\approx 0.7$-decade token span that is a
$1.3$--$1.4\times$ swing for LoRA (which is small enough that the $10^{-3}$ default stays in the basin everywhere) but a $1.8$--$2.1\times$ swing for FullFT, enough that token-heavy datasets have a higher full fine-tuning optimum.

Best achievable NLL falls monotonically with model size on every dataset in both families, and on a shared total-parameter axis the families interleave: a Qwen and a Llama of similar size reach similar loss (Figure~\ref{fig:scaling-xfamily}). Achievable loss depends on dataset and scale, with no separate family effect.

\begin{figure}[htbp]
  \centering
  \includegraphics[width=\linewidth]{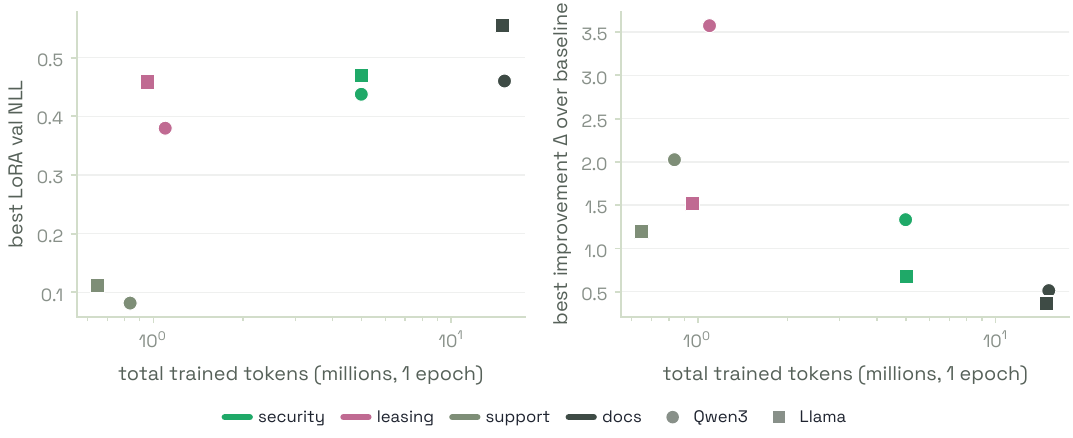}
  \caption{\textbf{Token load sets a per-token loss floor.} More trained tokens give worse
  absolute NLL (left) and smaller improvement over baseline (right) in both families (Qwen
  $\bullet$, Llama $\blacksquare$; colour = dataset). The families sit at different token counts
  because their tokenizers differ, but share the ordering.}
  \label{fig:tokens-floor}
\end{figure}

\begin{figure}[htbp]
  \centering
  \includegraphics[width=\linewidth]{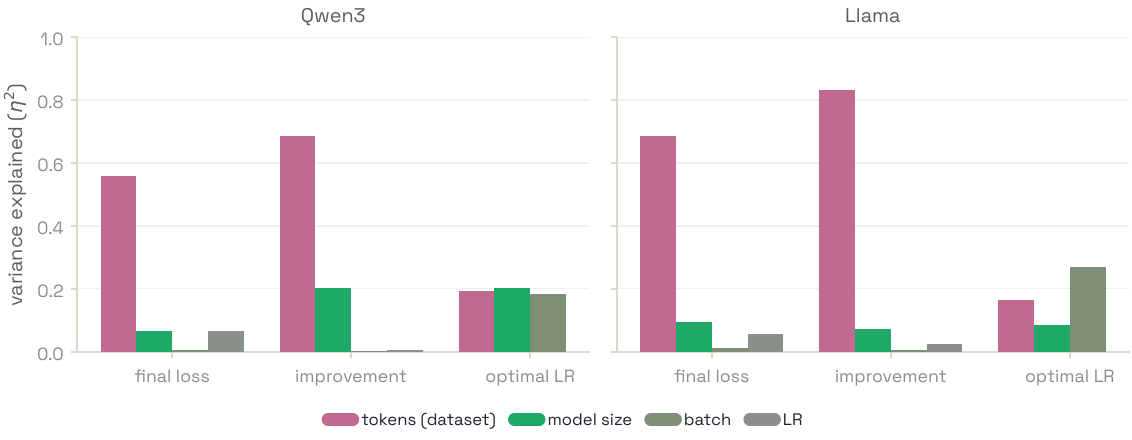}
  \caption{\textbf{Dataset composition dominates the loss; batch and learning rate barely move
  it.} Share of variance ($\eta^2$) in final NLL, improvement over baseline, and discrete-best
  learning rate, attributed to dataset (token composition), model size, global batch, and learning
  rate (LoRA, per family).}
  \label{fig:vardecomp}
\end{figure}

\begin{figure}[htbp]
  \centering
  \includegraphics[width=\linewidth]{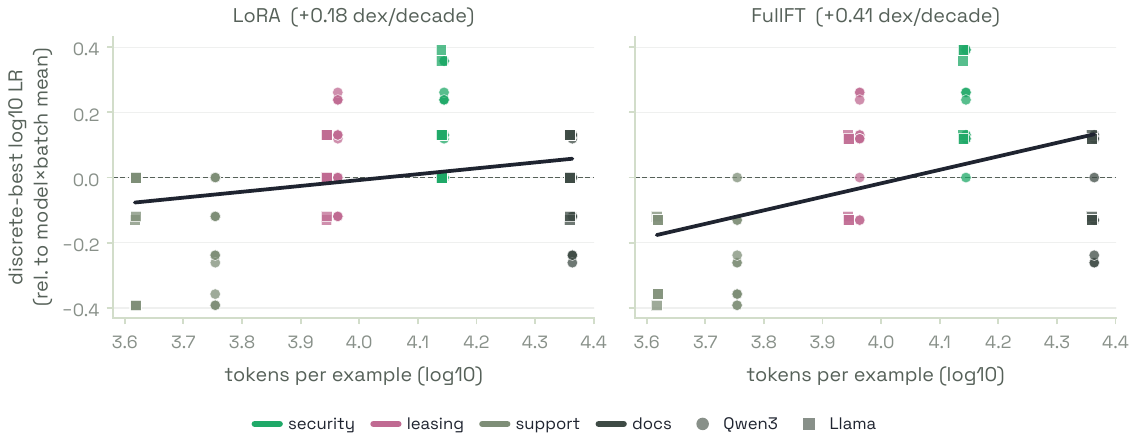}
  \caption{\textbf{The optimal learning rate drifts up with token load, about twice as steeply for
  FullFT as for LoRA.} Discrete-best learning rate, relative to its model$\times$batch mean so
  only the cross-dataset signal remains, against tokens per example, pooled over both families
  (Qwen $\bullet$, Llama $\blacksquare$; colour = dataset), for LoRA (left) and FullFT (right).}
  \label{fig:tok-optlr}
\end{figure}

\begin{figure}[htbp]
  \centering
  \includegraphics[width=\linewidth]{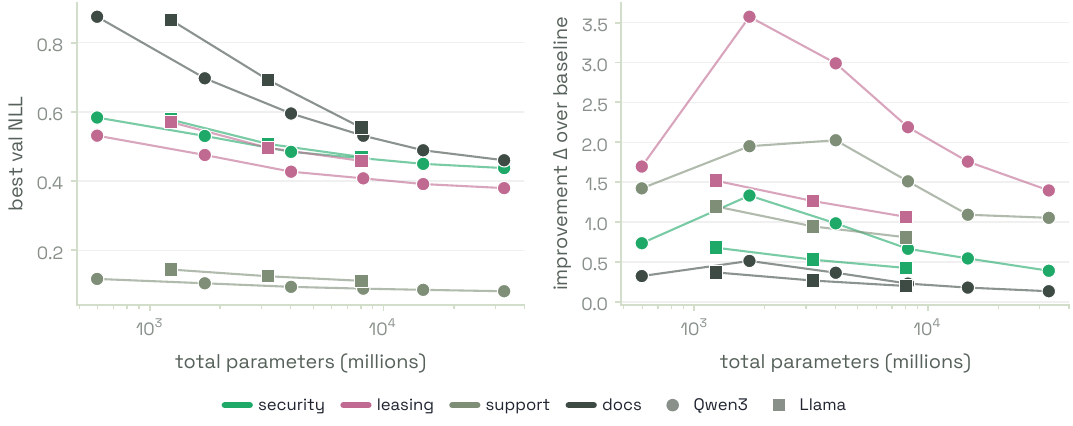}
  \caption{\textbf{The families interleave on a shared parameter axis.} Best loss (left) and
  improvement over baseline (right) against total parameters (Qwen $\bullet$, Llama
  $\blacksquare$; colour = dataset) follow the same within-dataset trend regardless of family.}
  \label{fig:scaling-xfamily}
\end{figure}

\subsection{High-learning-rate instability}
\label{subsec:instability}

The usable learning-rate range is bounded above by stability. Every run whose
validation loss ends materially above its running minimum sits at the top of the grid
($3\!\times\!10^{-3}$), concentrated on the larger models such as Qwen3-4B through 32B and Llama-3.2-3B ($24$ of $682$ eligible LoRA runs; Figure~\ref{fig:instability}). The $38$ excluded cells show the same pattern more strongly: at $3\!\times\!10^{-3}$ the largest Llama, Llama-3.1-8B, fails to beat its baseline in $16$ cells. Larger models have lower optima, so a fixed top-of-grid learning rate sits furthest above them; the usable ceiling is model-dependent, and the failure mode replicates across families.

\begin{figure}[htbp]
  \centering
  \includegraphics[width=0.8\linewidth]{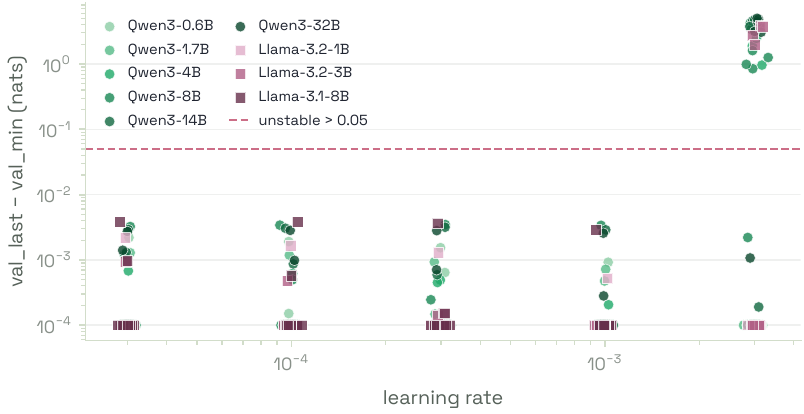}
  \caption{\textbf{Every unstable run sits at the top learning rate, and the larger models
  destabilise first.} End-of-training instability (\texttt{val\_last} $-$ \texttt{val\_min})
  against learning rate for LoRA (colour = model, Qwen $\bullet$ / Llama $\blacksquare$).}
  \label{fig:instability}
\end{figure}

\subsection{Learning-rate schedule}
\label{subsec:lrschedule}

The foundation sweep used cosine decay. To test whether that schedule is merely incidental, we reran Qwen3-8B LoRA on the same $5{,}000$-example, one-epoch setup with the same rank-64,
$\alpha=32$ adapter, but used a constant learning rate. The matched grid covers security and leasing, with global batch $\{8,16,32\}$ and learning rate
$\{10^{-4},3\!\times\!10^{-4},10^{-3},3\!\times\!10^{-3}\}$, comparing each constant run to its cosine-scheduled counterpart.

Constant learning rate is useful only when the chosen learning rate is too low. At $10^{-4}$ it improves all six matched cells by a mean $0.0175$ nats; at $3\!\times\!10^{-4}$ it is effectively tied (mean $0.0019$ nats better, with mixed signs). At the calibrated $10^{-3}$ LoRA default, cosine wins all six cells by a mean $0.0448$ nats. At $3\!\times\!10^{-3}$ the distinction is no longer a small optimisation choice but a stability issue: constant runs end $1.7$--$5.0$ nats above their best checkpoint, and the final loss is worse than cosine by a mean $2.03$ nats even when a mid-training checkpoint briefly looks competitive.

\begin{figure}[htbp]
  \centering
  \includegraphics[width=\linewidth]{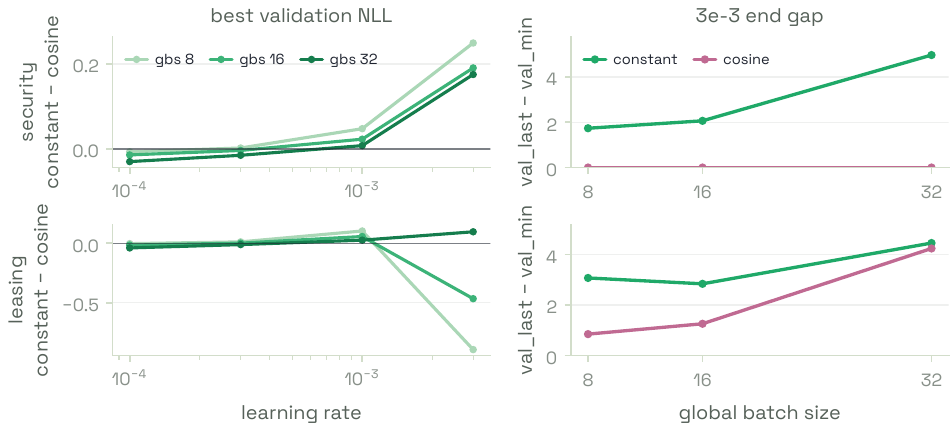}
  \caption{\textbf{Constant learning rate helps below the tuned LR but destabilises the high-LR
  edge.} Left: best-checkpoint validation NLL difference, constant minus cosine, on the two
  complete matched datasets. Negative values favour constant. Right: end-of-training gap
  (\texttt{val\_last} $-$ \texttt{val\_min}) at $3\!\times\!10^{-3}$, where cosine acts as the
  stability guardrail.}
  \label{fig:lr-schedule}
\end{figure}

\subsection{Out-of-architecture transfer: mixtures-of-experts}
\label{subsec:moe}

The results so far are all dense transformers. As a held-out architecture test we apply the rule to three Qwen3 mixtures-of-experts. Qwen3-30B-A3B-Instruct-2507 (hidden size $2048$ as Qwen3-1.7B; 48 layers, 128 experts, 8 active; $30.5$B total, $3.3$B active) gets the full sweep, reusing the datasets, splits, label policy, and LoRA grid, with FullFT over batches $\{16,64\}$ and a grid shifted to $\{3\!\times\!10^{-6},\dots,10^{-4}\}$, and never enters the law fits; $79$ of its $80$ LoRA cells and all $32$ FullFT cells are eligible, and rank-64 all-linear LoRA adapts every expert, training $\approx\!2.6$B parameters ($8.4\%$ of total). The dense selection rule transfers: the flat $10^{-3}$ LoRA rate is the discrete best in $13$ of $16$ (dataset, batch) cells at a mean regret of $0.0011$ nats, below the dense families' own $0.0026$ (Figure~\ref{fig:moe-frontiers}); the FullFT optimum is $3\!\times\!10^{-5}$ in six of eight cells, matching the dense law's $3.8\!\times\!10^{-5}$ prediction at a cost of $0.0050$ nats (dense $0.0052$); LoRA retains $97$--$100\%$ of FullFT's improvement; and instability appears only at $3\!\times\!10^{-3}$. Inverting the dense best-loss curve at the MoE's loss (Figure~\ref{fig:moe-equiv}) puts its dense-equivalent size at a median $8.5$B, far above its $3.3$B active count and close to the geometric mean of active and total ($10.1$B), so the 30B-A3B reaches roughly 8B-dense quality at a 3B active cost. The defaults then transfer blind to two larger models. Qwen3-Next-80B-A3B-Instruct interleaves gated-DeltaNet and gated-attention layers (3:1 over 48 layers)\footnote{The LoRA cells use the same \textsf{all-linear} target as the dense runs, so rank-64 adapters sit on every linear projection in the model: the query, key, value, and output projections of the 12 gated-attention layers; the fused input projections (\textsf{qkvz} and the beta/alpha projection) and the output projection of the 36 gated-DeltaNet layers; and the gate, up, and down projections of all 512 routed experts and the shared expert in each of the 48 MoE blocks. The non-linear parts of the DeltaNet path, its short causal 1D convolution and its gating and normalisation, are not adapted, since \textsf{all-linear} matches only linear layers.} with 512 experts ($79.7$B total, $3.3$B active); one blind LoRA cell per dataset at the dense defaults beats the tuned 30B on three of four datasets and falls below the dense 32B best on the two token-heavy ones, so its dense-equivalent rails at the top of the ladder, while rank-64 all-linear LoRA here trains $\approx\!8.9$B parameters ($11.2\%$ of total); its FullFT cells, first mistuned at $10^{-4}$ ($2.6\times$ above the rule), match LoRA to within $0.003$ nats once moved to the on-rule $4\times10^{-5}$. Qwen3-235B-A22B-Instruct ($235.1$B total, $22.2$B active) gets the same blind treatment, its LoRA cells eligible and lower-loss than the 80B on three of four datasets, continuing the monotone size trend, with FullFT at $4\times10^{-5}$ within $0.002$--$0.016$ nats of LoRA. Where the three MoEs sit on the dense size law, and the geometric-mean placement that fits all of them, is taken up in \S\ref{subsec:sizelaw}.

\begin{figure}[htbp]
  \centering
  \includegraphics[width=\linewidth]{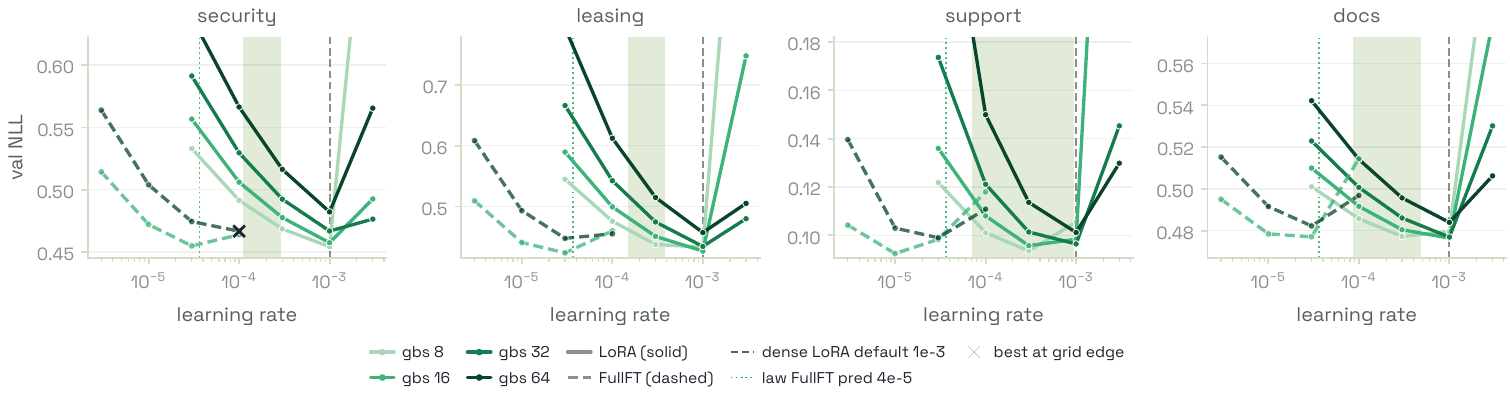}
  \caption{\textbf{The dense selection rule transfers to the MoE, for both tuners.} Validation NLL
  against learning rate for Qwen3-30B-A3B, one panel per dataset, one line per global batch (LoRA
  solid, FullFT dashed). The dashed vertical line is the dense LoRA default $10^{-3}$; the dotted
  line is the dense law's FullFT prediction at hidden size $2048$; $\times$ marks grid-edge
  optima.}
  \label{fig:moe-frontiers}
\end{figure}

\begin{figure}[htbp]
  \centering
  \includegraphics[width=\linewidth]{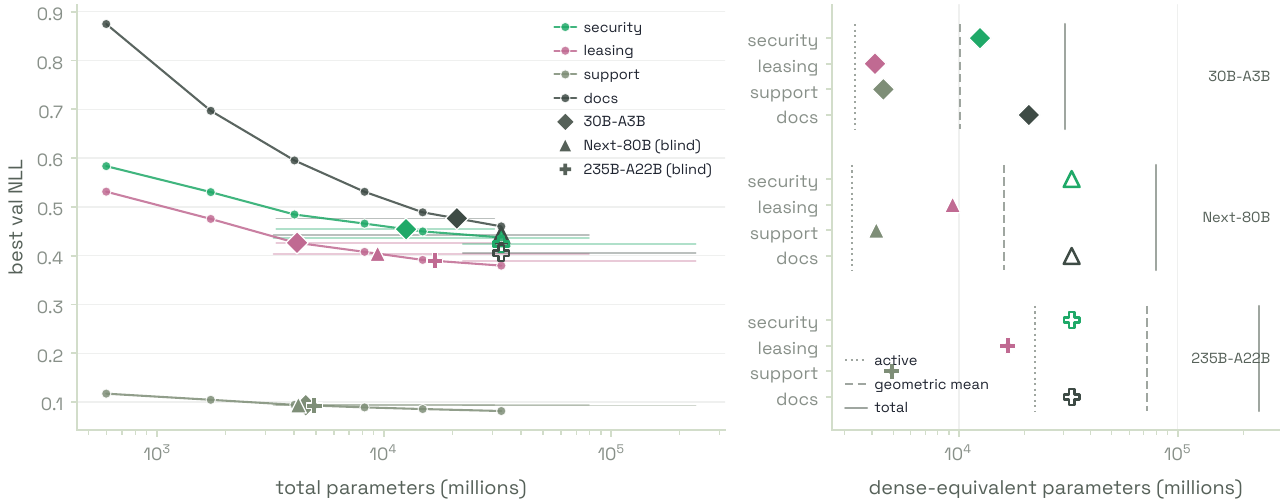}
  \caption{\textbf{All three MoEs fine-tune like dense models well above their active size on the
  token-heavy tasks: the 30B at a median $8.5$B, the 80B and 235B at or beyond the dense 32B on
  \textsf{security} and \textsf{docs}.} Left: dense Qwen3 best-NLL-vs-parameters curves (colour =
  dataset) with each MoE placed at its dense-equivalent size (30B-A3B diamonds, Next-80B triangles,
  235B-A22B pluses; open markers are railed lower bounds); the faint horizontal guides span each
  MoE's active to total parameters at its loss. Right: the per-dataset equivalent sizes against each
  MoE's three candidate placements (active $3.3$B for the 30B and 80B, $22.2$B for the 235B; totals
  $30.5$B, $79.7$B, and $235.1$B).}
  \label{fig:moe-equiv}
\end{figure}

%% file: sections/lora_hparams.tex
\section{LoRA Hyperparameters: Rank and Alpha}
\label{sec:lorahp}

The foundation sweep held the adapter fixed (rank $64$, $\alpha=32$, all-linear) and tuned only
the learning rate and batch. This section varies the adapter itself. Rank and $\alpha$ are not
independent: the standard $\alpha/r$ scaling shrinks the effective update as rank grows,
which collapses gradients at large rank unless corrected to $\alpha/\sqrt{r}$
\citep{kalajdzievski2023rslora}; the optimal learning rate can move with rank; and full
fine-tuning learns updates of $10$--$100\times$ higher rank than typical adapters
\citep{biderman2024lora}, while
\citet{zhang2024finetuning} report that adding rank barely moves the loss at all. We sweep rank
against $\alpha$ at fixed model, data, and learning-rate policy to test whether rank adds capacity
these tasks can use, and whether $\alpha$ does anything beyond rescaling the learning rate.

\subsection{Experimental setup}

The sweep uses Qwen3-4B and Qwen3-8B. We train on two datasets, \textsf{leasing} and \textsf{docs}, with $5{,}000$ examples for one epoch under the same
assistant-only masking, split policy, and no-truncation context filter used in
\S\ref{sec:foundation}. The learning rate is fixed at $10^{-3}$, the global batch is fixed at the foundation optimum for each dataset (8 for \textsf{leasing}, 16 for \textsf{docs}).

The rank sweep holds $\alpha=32$ and uses ranks
$8$, $16$, $32$, $64$, and $128$; the alpha sweep holds $r=64$ and uses $\alpha=16$, $32$,
and $64$. The shared anchor is the foundation default, $r=64,\alpha=32$.
All cells use all-linear LoRA, zero dropout, standard LoRA scaling ($\alpha/r$), and AdamW. We report losses both against the pretrained baseline and against the anchor within the same
model--dataset pair, so the readout is about the adapter geometry rather than dataset difficulty.

The question here is whether the default adapter should stay at $r=64,\alpha=32$, or whether a cheaper or larger rank or a different scale reliably lowers validation loss at the same data budget.

\subsection{Rank eventually plateaus}

 The rank response is consistent across both datasets and both model sizes
(Figure~\ref{fig:lora-rank}). Rank $8$ is too small: at $\alpha=32$ it is
$0.006$--$0.010$ nats worse than the $r=64,\alpha=32$ anchor. Rank $16$ is still
$0.003$--$0.006$ nats worse. Rank $32$ closes most of the gap, coming within
$0.0009$--$0.0028$ nats of the anchor while using about half the trainable parameters
($2.25\%$ rather than $4.40\%$ on Qwen3-4B, $1.56\%$ rather than $3.07\%$ on Qwen3-8B).

The plateau starts at rank $64$. Rank $128$ is the best validation-loss cell in three of the four model--dataset pairs, but only by $0.0004$--$0.0008$ nats; on Qwen3-8B/\textsf{leasing} it is slightly worse than rank $64$ by $0.0004$ nats. That extra rank roughly doubles the trainable parameter fraction relative to the anchor ($4.40\%\to8.43\%$ for 4B and
$3.07\%\to5.96\%$ for 8B). For this one-epoch, $5$k-example regime the useful capacity threshold is therefore around $r=64$, not $r=128$.

\begin{figure}[htbp]
  \centering
  \includegraphics[width=\linewidth]{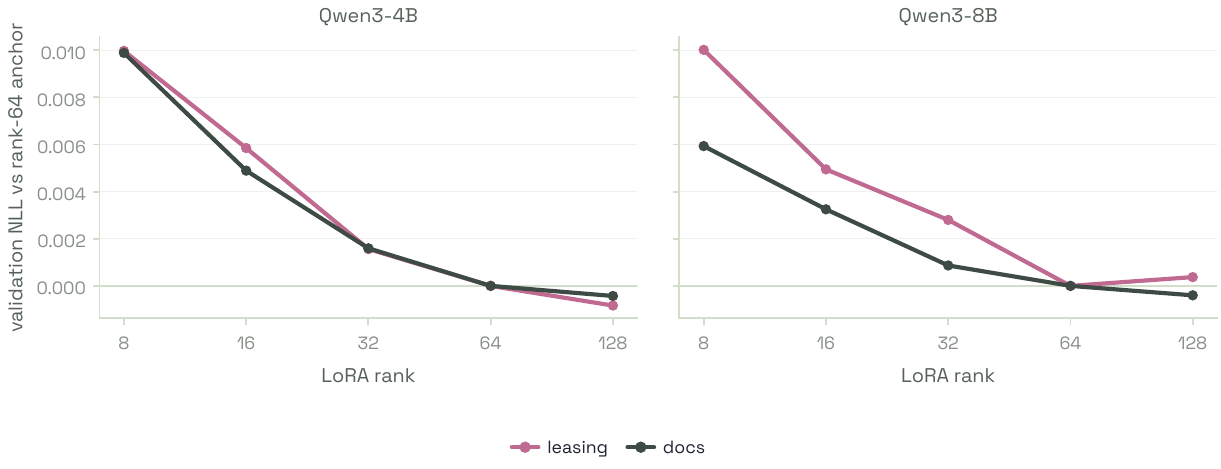}
  \caption{\textbf{Rank helps until about $64$, then mostly saturates.} Validation NLL relative
  to the $r=64,\alpha=32$ anchor, with $\alpha=32$ fixed. Rank $8$ and $16$ are clearly
  under-capacity, rank $32$ is close, and rank $128$ only gives sub-$0.001$ gains in three of four
  cells.}
  \label{fig:lora-rank}
\end{figure}

\subsection{Alpha sets the update scale}

At fixed rank $64$, $\alpha=32$ is best in every comparison (Figure~\ref{fig:lora-alpha}). Lowering to
$\alpha=16$ hurts by $0.0022$--$0.0082$ nats, with the largest penalty on
\textsf{leasing}. Raising to $\alpha=64$ also hurts, but less, by $0.0009$--$0.0037$ nats. Thus, at the optimal learning rate of $10^{-3}$, changing $\alpha$ does not reveal a better scale; it just moves the effective adapter step away from the best point.

This also explains why the rank result does not endorse arbitrarily higher rank under standard LoRA scaling. With $\alpha$ fixed, the LoRA multiplier $\alpha/r$ shrinks as rank grows, while the parameter count rises linearly. The observed curve says that the added capacity from $8\to64$ outweighs the smaller multiplier, but by rank $128$ the remaining loss improvement is too small to justify changing the default without a downstream reason.

\begin{figure}[htbp]
  \centering
  \includegraphics[width=\linewidth]{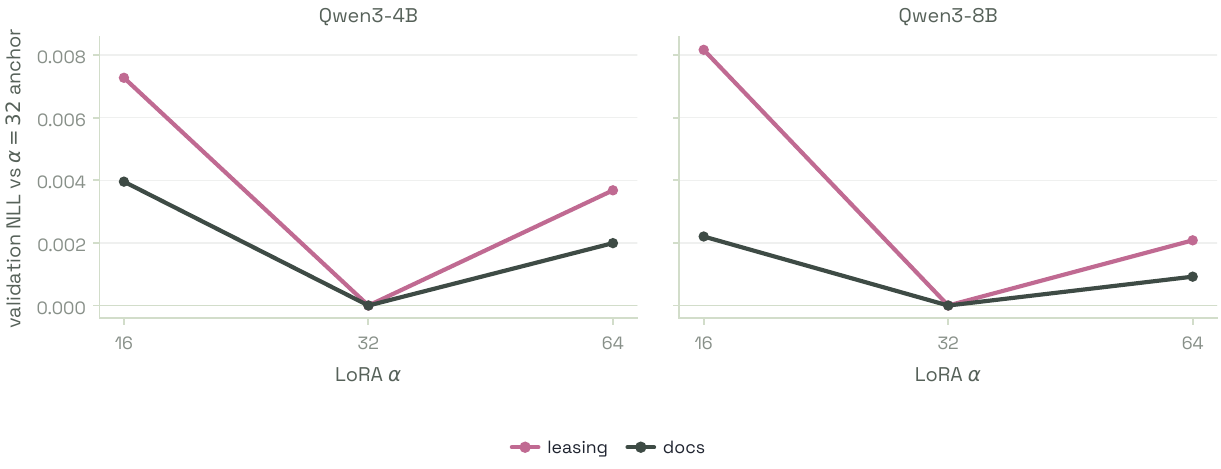}
  \caption{\textbf{At rank $64$, $\alpha=32$ is the best observed scale in all four cells.}
  Validation NLL relative to the $r=64,\alpha=32$ anchor. Both lower and higher alpha move away
  from the optimum at the fixed $10^{-3}$ learning rate.}
  \label{fig:lora-alpha}
\end{figure}

Note that we did not sweep learning rate jointly with rank, did not run a rank-stabilised LoRA variant, and did not run downstream evaluations for these cells. 

\begin{figure}[htbp]
  \centering
  \includegraphics[width=\linewidth]{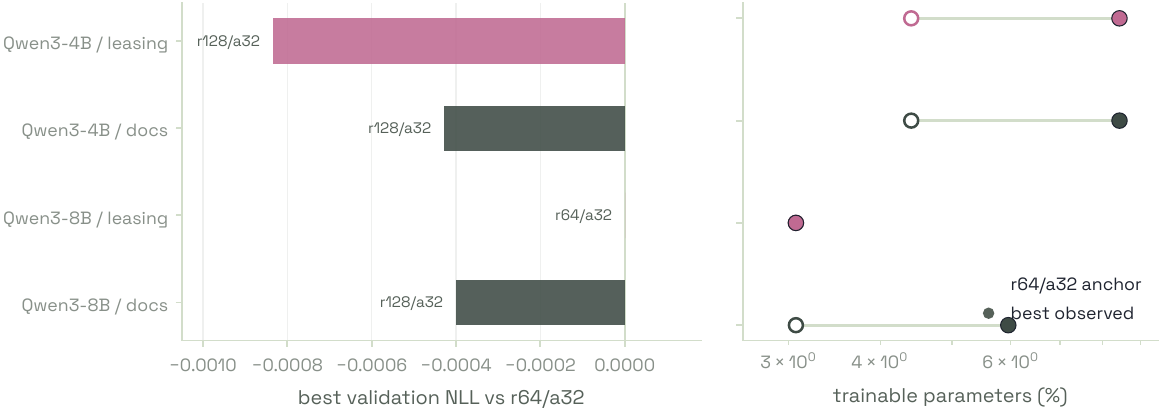}
  \caption{\textbf{The best observed cells barely improve on the rank-$64$ default, while
  rank-$128$ roughly doubles adapter size.} Left: validation NLL of the best observed
  rank/alpha setting relative to $r=64,\alpha=32$ for each model--dataset pair. Right: trainable
  parameter fraction for the anchor (open marker) and best observed cell (filled marker).}
  \label{fig:lora-cost}
\end{figure}

%% file: sections/downstream_performance.tex
\section{Predicting Downstream Performance from Loss}
\label{sec:downstream}

Every selection in this report is made by validation loss. That is legitimate only if loss
ranks what we care about: performance on our curated customer-task
evaluations. This section tests the assumption directly.

The literature supports loss as a predictor, with limits. Downstream performance is largely a
function of loss rather than of size or compute \citep{du2024emergentloss}, much apparent
``emergence'' dissolves under smooth metrics \citep{schaeffer2023mirage}, and aggregate downstream error is reliably predictable from loss \citep{gadre2024scaling, owen2024predictable, saunshi2021mathematical}. But per-task predictability is far noisier than the aggregate \citep{owen2024predictable, lourie2025unreliable}, the task metric can worsen while cross-entropy falls under distribution shift \citep{isik2024scaling}, and \citet{liu2023implicitbias} exhibit models at \emph{identical} loss with different downstream accuracy, attributing the gap to the optimiser's implicit bias toward flat minima, a bias whose strength scales with the learning-rate-to-batch-size ratio \citep{li2022sgdzeroloss, damian2021labelnoise, smith2021origin}, the two levers swept in \S\ref{sec:foundation}.

We therefore measure two quantities per run: the validation loss we already select on, and the flatness of the converged loss, summarised by the Hessian trace. At a well-fit minimiser that trace equals the trace of the Fisher information: the expected squared gradient norm of a single sampled-token log-probability, computed with one backward pass per sample and no second derivatives (Appendix~\ref{app:hessian}). The known failure modes of sharpness measures \citep{andriushchenko2023modern, granziol2020false, dinh2017sharp, flatlora2025} are handled by design: the trace is read within matched learning-rate and batch cells so recipe effects cannot confound flatness, LoRA curvature is reported as subspace curvature, and flatness is reported alongside loss.

\paragraph{Design.} For every eligible run we record validation NLL, the Fisher-trace
flatness of the trainable parameters, and graded performance on the held-out customer-task
evaluations. We ask three questions: does validation loss rank downstream quality within
(model, dataset) groups; does flatness explain the residual among runs at matched loss; and does the measured trace move with batch and learning rate as the implicit-bias theory predicts. The flatness hypothesis is falsified if, at matched loss, the trace adds no ranking information.

We have $896$ cell-by-task pairs across the $224$ cells. Within matched (model, dataset, recipe) cells, validation loss ranks judged quality on both models and tuners
(within-dataset-standardised Spearman $\rho_s$ from $-0.38$ to $-0.88$). The Fisher trace falls with the learning-rate-to-batch ratio as the implicit-bias account predicts, yet at matched loss it adds no reliable ranking and is not consistently signed, so the flatness hypothesis is not supported at the fine-tuned scale.

\subsection{Calibrating the instruments on base models}
\label{subsec:base-calibration}

Before the fine-tuned grid is scored, we calibrate both measurements on the nine base models. Each dataset has its own production evaluation: an LLM-judged instruction-following pass rate (support), a judged critical-error pass rate (leasing), a graded markup-and-writing score (docs), and a findings-comparison micro-F1 (security). We serve every base model at the full $40{,}960$-token context, regenerate the final assistant turn of the first $200$ canonical-test rows per dataset, and grade with GPT-5.5 under each dataset's judge. On the same $200$ rows we compute each base model's test NLL through the identical forward-pass and label-masking path that produces every validation loss in this report, and its Fisher trace: $64$ single-backward samples of the squared gradient norm of a sampled-token log-probability (Appendix~\ref{app:hessian}). Note that base models are nowhere near the well-fit regime in which the Bartlett identity turns this into a Hessian trace; and on a fresh base load every parameter carries gradients, so the trace is full-parameter here versus LoRA-subspace for the fine-tuned runs to come.

\begin{figure}[htbp]
  \centering
  \includegraphics[width=\linewidth]{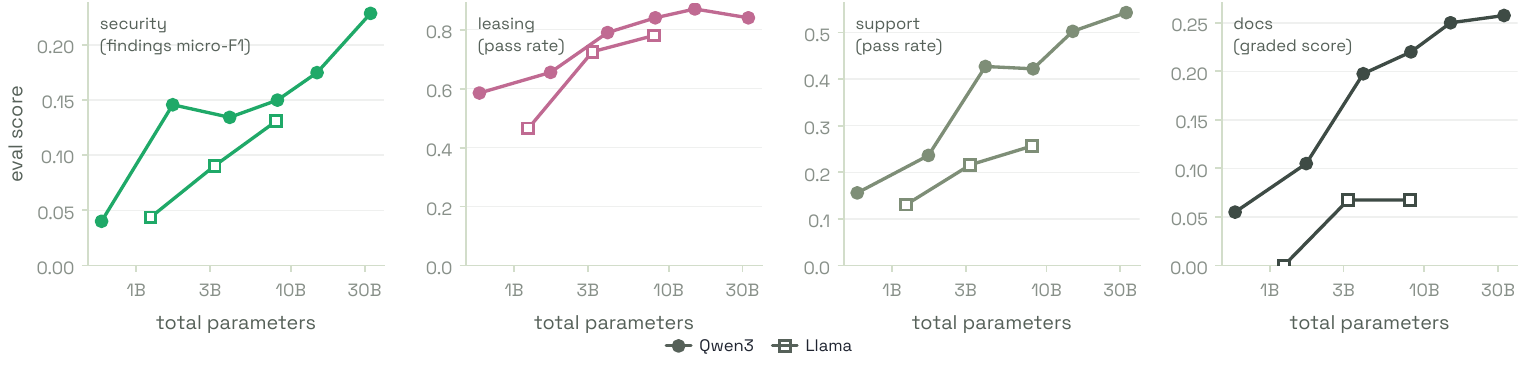}
  \caption{\textbf{The production judges resolve base-model quality before any fine-tuning.}
  Eval score against total parameters per dataset (Qwen $\bullet$, Llama $\square$).}
  \label{fig:base-calib-size}
\end{figure}

Within each family, scores rise near-monotonically with scale on all four tasks (Figure~\ref{fig:base-calib-size}), so the evaluations rank models well below the quality range fine-tuning will occupy.

\begin{figure}[htbp]
  \centering
  \includegraphics[width=\linewidth]{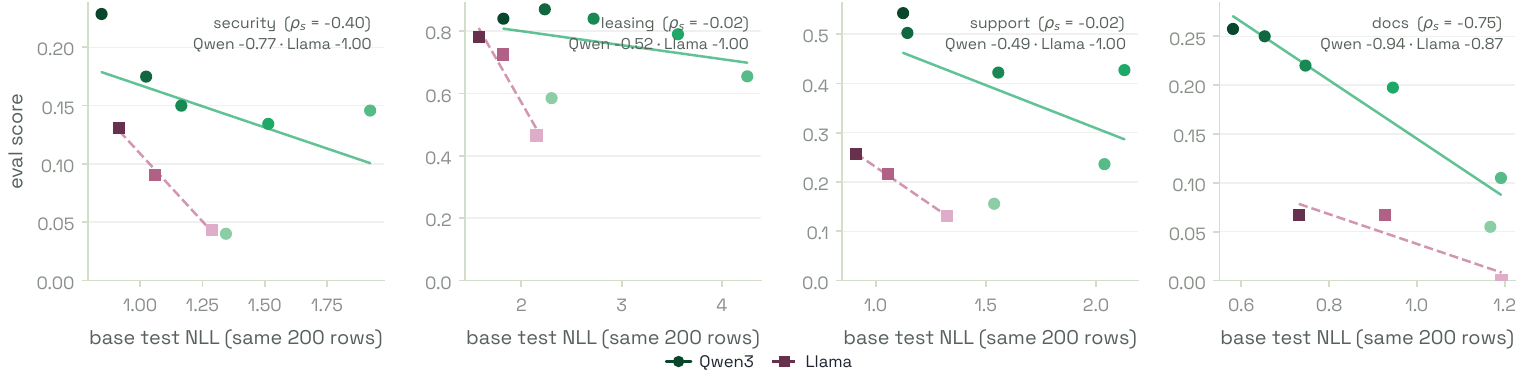}
  \caption{\textbf{Within a family, lower test NLL ranks judged quality on every dataset; pooling
  the families hides it.} Eval score against base test NLL on the same $200$ rows (Qwen $\bullet$,
  Llama $\blacksquare$; light~$\to$~dark with model scale), with a least-squares fit per family
  (Qwen solid, Llama dashed) and Spearman $\rho_s$ pooled and per family. Llama fits rest on three
  models, so its $\rho_s$ is nearly degenerate and shown for completeness.}
  \label{fig:base-calib-nll}
\end{figure}

The loss--eval calibration splits along family lines (Figure~\ref{fig:base-calib-nll}). Pooling all nine models the correlation is weak ($-0.35$ after within-dataset standardisation), but within each family it is negative on every dataset, so the near-zero aggregate is a family-mixture effect rather than an absence of signal. Llama reaches \emph{lower} NLL than Qwen at matched size on the agentic datasets while scoring \emph{lower} with the judges, so fit-to-text and judged behaviour separate, and any loss-based selection that crosses families inherits that gap. With six and three models per fit these are indicative; the fine-tuned grid re-runs the same pipeline with hundreds of runs per dataset.

\begin{figure}[htbp]
  \centering
  \includegraphics[width=\linewidth]{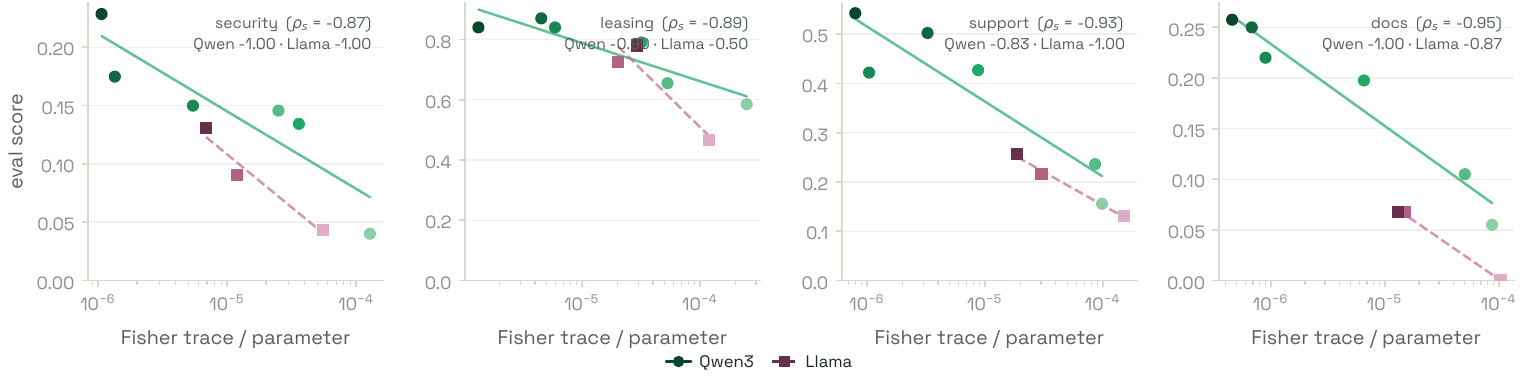}
  \caption{\textbf{The base-model Fisher trace ranks judged quality across both families at once.}
  Eval score against the per-parameter Fisher trace on the same $200$ rows (log axis; Qwen
  $\bullet$, Llama $\blacksquare$; light~$\to$~dark with model scale), least-squares fit per
  family (Qwen solid, Llama dashed) and Spearman $\rho_s$ pooled and per family. Full-parameter
  scope; $16$ examples $\times$ $4$ probes per cell.}
  \label{fig:base-calib-trace}
\end{figure}

The Fisher trace is the stronger of the two calibrations (Figure~\ref{fig:base-calib-trace}): lower trace accompanies higher judged quality on every dataset ($-0.90$ pooled, against $-0.35$ for test NLL), and Qwen and Llama sit on one common trend rather than splitting by family. The caveat matters: across bases the per-parameter trace is strongly size-correlated ($\rho_s=-0.95$ with total parameters), so the $-0.90$ largely re-expresses scale, and regressing out size leaves a partial of only $\rho_s\approx-0.15$. The fine-tuned grid resolves this by fixing model size: at matched loss within a (model, dataset, recipe) cell, does the trace still rank the judges' verdicts?

\subsection{The fine-tuned grid}
\label{subsec:finetuned-grid}

The grid tests the within-recipe question the base models could only anchor. We score the $224$
completed foundation cells (Qwen3-8B and Llama-3.1-8B, LoRA and full fine-tuning, across the swept learning rates and batch sizes) on each dataset's production judge, in-domain, regenerating the final assistant turn of the same $200$ canonical-test rows used for the base calibration. For every cell we read its final test NLL on those rows and re-measure the Fisher trace to be base-comparable: $16$ examples $\times$ $4$ probes $=64$ single-backward samples, on the test split, over the
trainable parameters (the LoRA adapter for adapter runs, the full model for full fine-tuning;
Appendix~\ref{app:hessian}). Holding (model, dataset, tuner) fixed and sweeping learning rate and
batch isolates the matched-loss question. We standardise eval, loss, and trace within each
(model, dataset) cell-set and pool the z-scores across the four in-domain datasets per (model,
tuner) group, so datasets on different metric and loss scales contribute on one footing. Four
divergent docs cells at the largest learning rate ($10^{-3}$) generate non-terminating output past
the $20\%$ judge-error gate and are excluded, leaving $220$ in-domain pairs.

\begin{figure}[htbp]
  \centering
  \includegraphics[width=\linewidth]{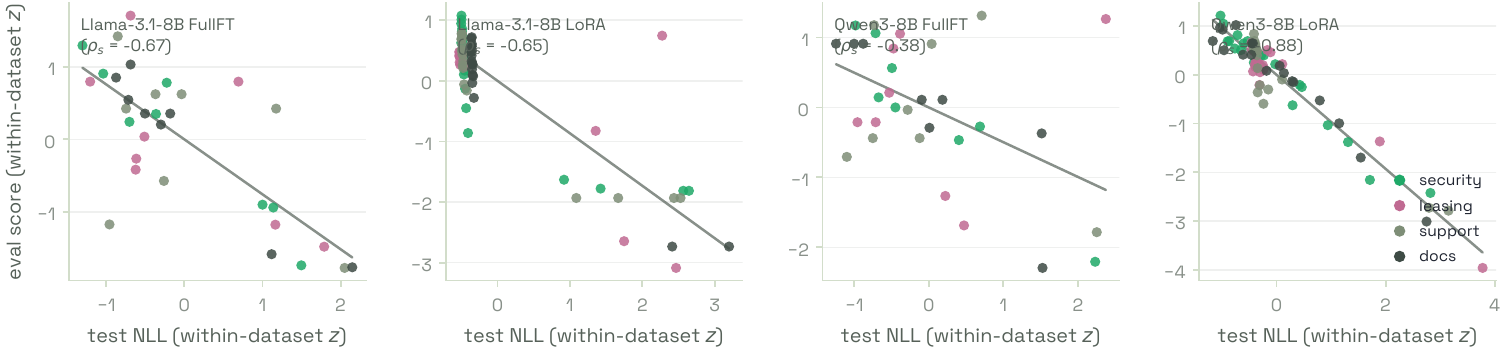}
  \caption{\textbf{Within a recipe, lower validation loss ranks higher judged quality on both models
  and tuners.} Eval score against test NLL, each standardised within its (model, dataset) cell-set
  and pooled across the four in-domain datasets (colour = dataset), per (model, tuner) group with a
  pooled least-squares fit and within-dataset-standardised Spearman $\rho_s$. Each point is one
  learning-rate$\times$batch cell.}
  \label{fig:ft-loss}
\end{figure}

Validation loss ranks judged quality within every recipe (Figure~\ref{fig:ft-loss}). Within each dataset, the standardised Spearman is negative in all four groups, from $-0.38$ on Qwen3-8B full fine-tuning to $-0.88$ on Qwen3-8B LoRA, with Pearson $r$ between
$-0.50$ and $-0.96$. The per-dataset splits agree on $14$ of the $16$ (group, dataset) cells; the two exceptions are Qwen3-8B full fine-tuning on leasing and support, where eight cells span a narrow loss range and the rank correlation is near zero ($+0.13$). With $32$ cells per full-fine-tuning group and $78$ per LoRA group, the within-recipe relation the nine bases could only suggest now holds: at fixed model, dataset, and tuner, lower validation loss accompanies higher judged quality, so selection by loss is sound inside a recipe.

\begin{figure}[htbp]
  \centering
  \includegraphics[width=\linewidth]{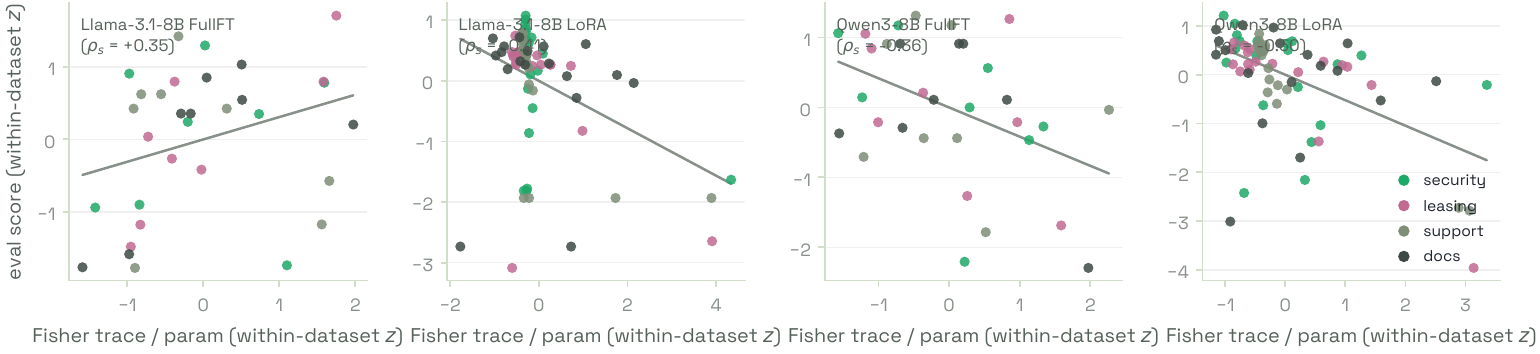}
  \caption{\textbf{At matched loss the Fisher trace does not rank judged quality.} Eval score against
  the base-comparable per-parameter Fisher trace, standardised and pooled as in
  Figure~\ref{fig:ft-loss}. The slope is weaker than for loss in every group and positive for
  Llama-3.1-8B full fine-tuning, so flatness adds no ranking information among runs at matched loss.}
  \label{fig:ft-trace}
\end{figure}

At matched loss the Fisher trace adds no reliable ranking (Figures~\ref{fig:ft-trace} and~\ref{fig:ft-summary}). Pooled within each group it is weaker than loss throughout and its sign is not stable, turning positive for Llama-3.1-8B full fine-tuning. A flatness measure carrying matched-loss information would rank eval in one direction across groups and at least as well as loss; the trace does neither. This is the falsification condition set in the design (\S\ref{sec:downstream}): within a recipe, the converged-minimum trace does not explain the residual in judged quality at matched loss.

\begin{figure}[htbp]
  \centering
  \includegraphics[width=0.7\linewidth]{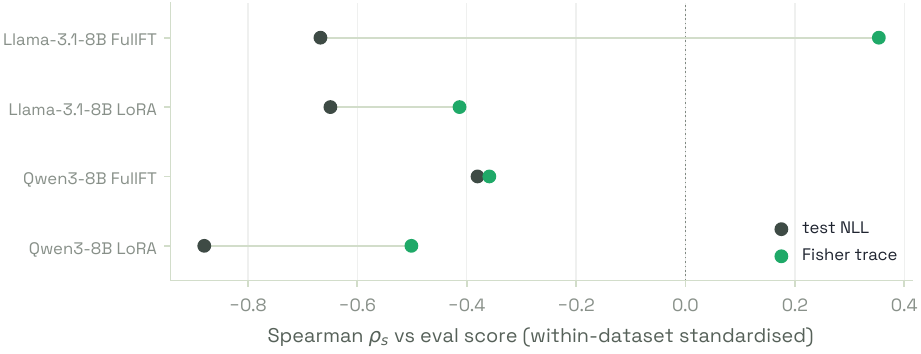}
  \caption{\textbf{Loss is the reliable within-recipe predictor; the trace is not.}
  Within-dataset-standardised Spearman $\rho_s$ against eval score for validation loss (graphite) and
  the Fisher trace (green), per (model, tuner) group. Loss lies left of zero in every group; the
  trace is weaker throughout and crosses to the wrong side of zero for Llama-3.1-8B full
  fine-tuning.}
  \label{fig:ft-summary}
\end{figure}

The contrast with the base calibration explains the reversal. The base trace's apparent edge was largely model size; fixing model and recipe removes the size axis, and the trace's ranking disappears while loss's persists. The trace remains a real flatness measure and moves as the implicit-bias account predicts \citep{liu2023implicitbias}: it falls with the learning-rate-to-batch ratio, strongly for full fine-tuning and weakly for LoRA, the ratio theory ties to flat-minima selection \citep{li2022sgdzeroloss, smith2021origin}. That dependence does not reach the judges, so for selection within a recipe validation loss is the measurement to use and the trace adds nothing reliable on top of it.

\subsection{Flatness tracks general-capability transfer}
\label{subsec:flatness-transfer}

The matched-loss test above scores the trace against in-domain quality. \citet{liu2023implicitbias}
make a narrower claim: among models at the \emph{same} pre-training loss, produced by interventions
that change the solution while holding loss fixed (continued training, a larger model, a different
algorithm), flatter minima \emph{transfer} better. A learning-rate and batch sweep at a fixed base
reproduces neither the matched-loss construction nor the transfer target, so a within-recipe null
does not bear on it. Recut against transfer, the link appears (Figure~\ref{fig:flatness-transfer}).
At genuinely matched loss, in-domain cell pairs within $0.01$ nats, the flatter cell has the higher
eval in $57$--$77\%$ of pairs (Qwen3-8B FullFT $17/22$, $p\approx0.008$), so the pooled-correlation
null above is partly a smoothing artefact, though the in-domain effect stays weak. Across the epoch
sweep (\S\ref{sec:epochs}), standardised within each (model, arm, dataset), a higher trace tracks
\emph{lower} general capability (trace--IFEval partial $\rho_s=-0.19$ at fixed loss) but
\emph{higher} in-domain task score ($+0.34$), so sharper minima accompany higher in-domain score and flatter minima higher general capability.
The optimiser change, the algorithm intervention of \citet{liu2023implicitbias}, is the clearest test:
across the eight validation-NLL-selected Muon and AdamW checkpoints the per-parameter trace ranks
IFEval at $\rho_s=-0.71$ (loss gives $-0.36$), with Muon flatter and higher on IFEval on all four
datasets. The flatness probe adds nothing to loss for selecting learning rate and batch on in-domain
quality, but it does track the general capability a fine-tune retains, most clearly when the
optimiser changes.

\begin{figure}[htbp]
  \centering
  \includegraphics[width=\linewidth]{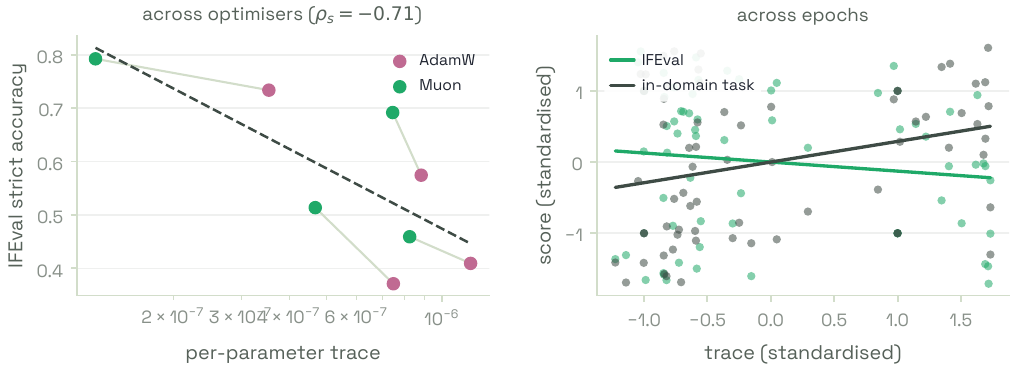}
  \caption{\textbf{Flatness tracks retained general capability rather than in-domain quality.}
  Left: across the eight validation-NLL-selected Muon (green) and AdamW (rose) checkpoints, IFEval
  rises as the per-parameter Fisher trace falls ($\rho_s=-0.71$); each line joins a dataset's pair,
  flatter and higher for Muon. Right: across the epoch sweep, standardised within (model, arm,
  dataset), a higher trace accompanies lower IFEval (green) but higher in-domain task score (graphite).}
  \label{fig:flatness-transfer}
\end{figure}

%% file: sections/scaling_laws.tex
\section{Scaling Trends in Post-Training}
\label{sec:scaling}

Pretraining obeys clean power laws in parameters, data, and compute \citep{kaplan2020scaling,
hoffmann2022chinchilla, bahri2024explaining}. Post-training has no single such law because it has
several axes at once: base-model size, supervised-data volume, the adapter's trainable-parameter
budget, and (a distinction pretraining can ignore) whether data is counted in examples or in
tokens. \S\ref{sec:foundation} showed best NLL falling monotonically with model size in both
families, but we could not fit a law with such a small number of points. This section asks which post-training
resources scale predictably on our grid, and how far any fit can be trusted. To do so, we include larger models, with three Qwen3 mixtures-of-experts (30B-A3B, Next-80B-A3B, and Qwen3-235B-A22B, \S\ref{subsec:moe}), which raises a question the dense ladder avoids: whether an MoE enters a fit at its total parameters, its active parameters, or its dense-equivalent size (median $8.5$B for the 30B-A3B; at least $32.8$B on two datasets for the
80B).

\citet{zhang2024finetuning} fit the closest prior law: fine-tuning loss decays as a power law in a
model-side factor and in example count, with model size scaling far more steeply than data. If
that holds here, a larger base model is worth more than additional labelled examples. They also
report that LoRA rank barely scales; we treat that as open, since the standard $\alpha/r$
convention collapses gradients at large rank \citep{kalajdzievski2023rslora} and full fine-tuning
learns far higher-rank updates \citep{biderman2024lora}; \S\ref{sec:lorahp} takes this up. The
unit question is the \S\ref{subsec:data} token floor restated: our datasets differ
$18$--$23\times$ in trained tokens at fixed example count, so an exponent fit in examples and one
fit in tokens need not agree.

Few-point fits also need care. The estimation literature puts the safe minimum
near five sizes \citep{choshen2024hitchhiker}, regime breaks are undetectable with few points
\citep{caballero2022broken}, and tight confidence intervals are unattainable
\citep{besiroglu2024chinchilla}. One historical hazard is avoided by construction: every cell uses
its own calibrated learning rate and batch (\S\ref{sec:foundation}), so the exponents are not
artefacts of mistuning the larger models \citep{porian2024resolving}.

\paragraph{Design.} The size-trend fit uses the Qwen3 ladder ($\ge 5$ sizes; Llama's three serve as a
transfer check), quoting ranges. LoRA and FullFT are fit separately,
testing whether the $98\%$ retention of \S\ref{sec:foundation} widens or closes with scale. MoEs
(30B-A3B, Next-80B-A3B, and 235B-A22B) enter at three candidate placements (total, active, and
geometric mean) and fit quality adjudicates between them. The data axis is measured as a minimal
Qwen3 size-by-data interaction grid: LoRA at the calibrated $10^{-3}$ learning rate and global
batch $8$, with $5$k, $10$k, and $15$k one-epoch examples on Qwen3-4B, 8B, and 14B for
\textsf{leasing} and \textsf{security}. That grid tests whether the local return to fresh
examples changes with model size, but it still does not separate examples from tokens within a
dataset.

\subsection{The model-size trend on the dense ladders}
\label{subsec:sizelaw}

The size axis is measurable now. The ladder takes one point per model: the eligible cell at the
\S\ref{sec:foundation} defaults (LoRA $10^{-3}$, FullFT $3\!\times\!10^{-5}$, global batch $16$),
fixed across families, so the fit describes what the recommended recipe achieves as size grows.
The choice costs little: against an oracle over the learning-rate grid at the same batch, the
defaults give up a mean $0.0008$ nats for LoRA and $0.0048$ for FullFT. Per dataset we fit a
saturating power law $L(N) = L_\infty + A\,N^{-\alpha}$ in total parameters on the six Qwen3
sizes, with pairs-bootstrap $95\%$ intervals ($B=1{,}000$); Llama's three sizes are scored
against the fit, never used in it.

Figure~\ref{fig:sizelaw} shows the fits and Table~\ref{tab:sizelaw} lists their parameters. Loss falls as a saturating power law in size on every dataset, with a LoRA exponent of $0.31$--$0.40$ and a steeper FullFT exponent of $0.40$--$0.51$, both with wide intervals at six points. The steeper FullFT decay restates the \S\ref{sec:foundation} retention result as a trend: the FullFT-minus-LoRA gap at matched defaults narrows from a mean $0.008$ nats at $0.6$B to $0.003$ at $32$B, so LoRA's $98\%$ retention closes with scale. The fitted floors track token load (the \S\ref{subsec:data} effect). The FullFT exponents are read as ranges, since the \textsf{support} exponent moves from $0.51$ to $0.33$ when re-fit on the oracle learning rate. Scored against the fit, Llama sits slightly above the Qwen fit at matched parameters ($23$ of $24$ residuals positive, most at 1B), an offset expected from its different tokenizer; the size trend holds.

Scoring each MoE's at-defaults loss against the dense fit at three a-priori placements (active, total, and their geometric mean) adjudicates which size an MoE fine-tunes like. The geometric mean minimises the mean absolute residual for every MoE under both tuners, at $0.016$--$0.022$ nats against $0.021$--$0.073$ for the active and total placements, and the pattern now holds up to the 235B-A22B (Figure~\ref{fig:moe-placement}). Two caveats: the 235B geometric-mean and total sizes extrapolate beyond the $0.6$--$32$B fitting range, and its FullFT cells sit at the closest observed $4\times10^{-5}$ rather than the $3\times10^{-5}$ rule. On a dense size trend, an MoE sits at roughly the geometric mean of its active and total parameters.
\begin{figure}[htbp]
  \centering
  \includegraphics[width=\linewidth]{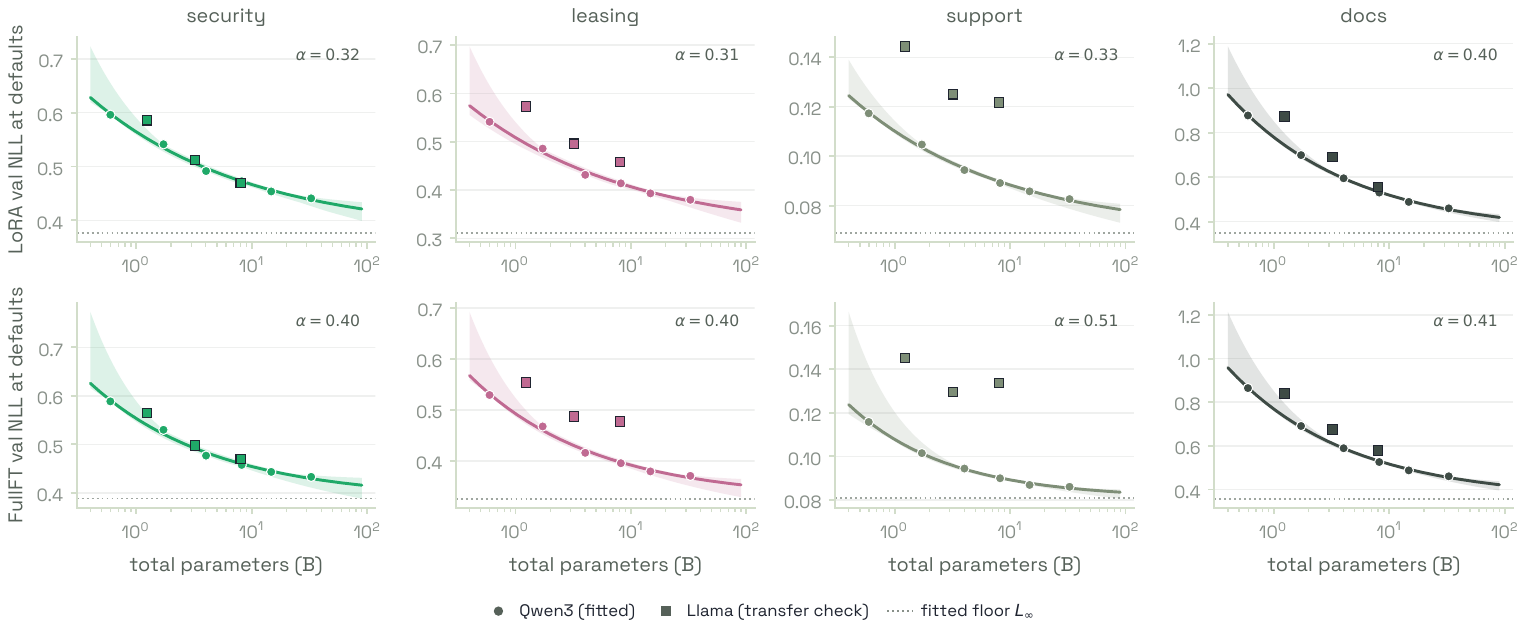}
  \caption{\textbf{At the fixed defaults, loss falls as a saturating power law in size, more
  steeply for FullFT than for LoRA.} Validation NLL at the \S\ref{sec:foundation} defaults
  against total parameters, one column per dataset (top LoRA, bottom FullFT): Qwen3 points with
  the fitted curve and its bootstrap band; Llama squares are scored against the fit, never used in
  it; the dotted line is the fitted floor $L_\infty$.}
  \label{fig:sizelaw}
\end{figure}

\begin{figure}[htbp]
  \centering
  \includegraphics[width=\linewidth]{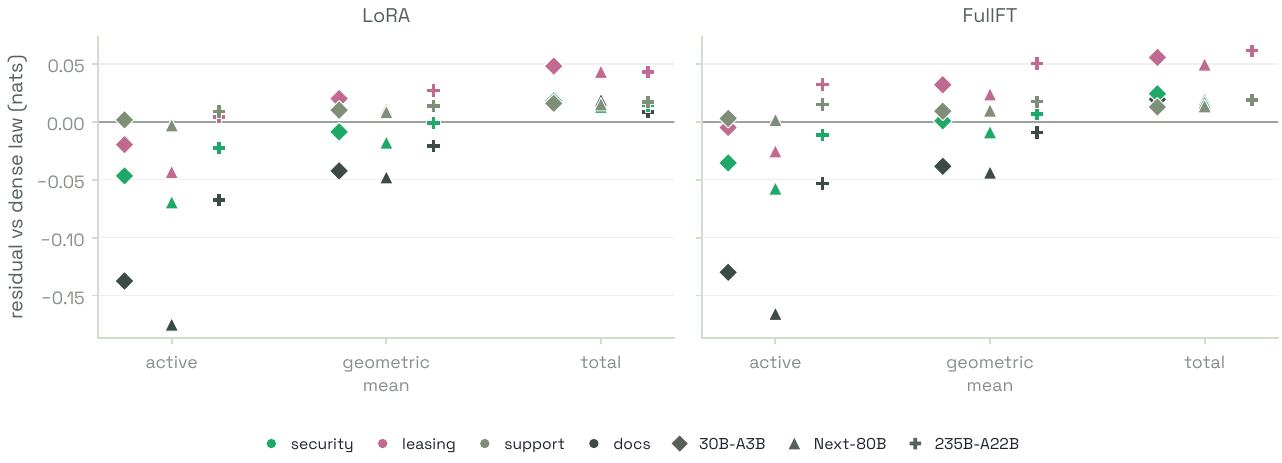}
  \caption{\textbf{All three MoEs enter the dense size trend at the geometric mean of their active
  and total parameters.} Signed residual of each MoE's at-defaults loss against the dense Qwen3 fit
  evaluated at each candidate placement (colour = dataset; diamond 30B-A3B, triangle Next-80B,
  plus 235B-A22B). Negative means the MoE beats a dense model of that size. Large-MoE FullFT cells
  use the closest observed LR to the $3\times10^{-5}$ rule: $4\times10^{-5}$ for both Next-80B and
  235B-A22B.}
  \label{fig:moe-placement}
\end{figure}

\subsection{Size by data: the return to fresh examples is dataset-specific}
\label{subsec:data-volume-law}

The fixed-8B arm showed that fresh examples lower the loss, but it left a serious ambiguity: maybe
additional data is only useful at one model size. The smallest interaction test adds the missing
Qwen3-4B and Qwen3-14B cells around the existing Qwen3-8B rows, holding the LoRA recipe fixed
($10^{-3}$, global batch $8$, one epoch) and varying only fresh examples on \textsf{leasing} and
\textsf{security}. This is still not a full scaling law in $N_{\mathrm{model}}$ and
$N_{\mathrm{data}}$; it is a local interaction check over the sizes and data volumes the report
can measure densely.

The example return is nearly parallel across sizes. Tripling fresh examples from $5$k to $15$k lowers validation NLL by about $20\%$ at all three sizes on \textsf{leasing} and about $8\%$ on \textsf{security}; across the six model--dataset curves the median gain is $14.4\%$ (Figure~\ref{fig:size-data-interaction}). The absolute loss drops with size while the fractional gain does not, so the fractional gain is size-independent. The dominant modifier is the dataset: \textsf{leasing} recovers about a fifth of the remaining NLL from tripling examples, \textsf{security} under a tenth. Examples and tokens stay confounded within a dataset, since token counts scale almost linearly with examples. When filtered in-domain examples exist, spend them before repeating data, and expect a dataset-specific gain that is stable across nearby Qwen3 sizes.

\begin{figure}[htbp]
  \centering
  \includegraphics[width=0.86\linewidth]{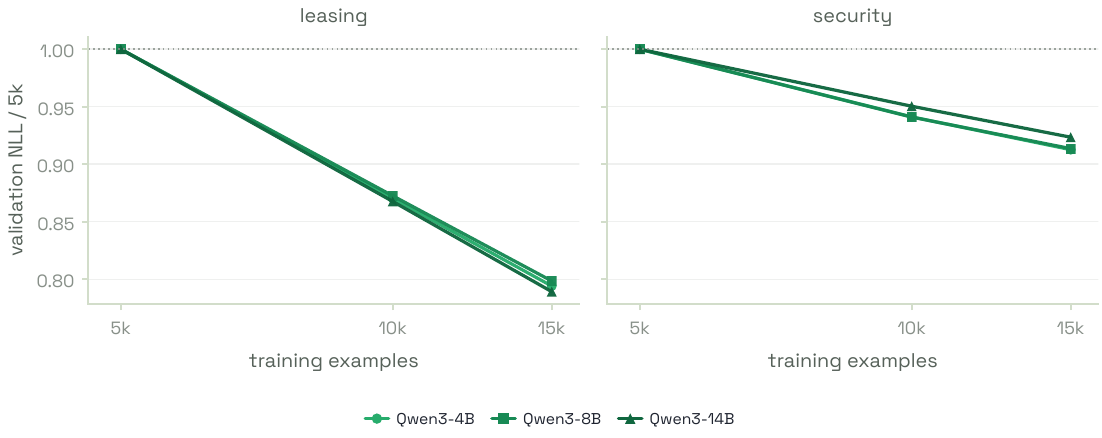}
  \caption{\textbf{Fresh examples help at 4B, 8B, and 14B; the size interaction is small relative
  to the dataset effect.} Validation NLL is normalised to each model--dataset cell's $5$k value.
  Lines are Qwen3 LoRA runs at the calibrated $10^{-3}$ learning rate and global batch $8$,
  trained for one epoch on $5$k, $10$k, and $15$k examples.}
  \label{fig:size-data-interaction}
\end{figure}

%% file: sections/optimizers.tex
\section{Optimisers}
\label{sec:optimizers}

AdamW is the standard optimiser for language-model training, a per-coordinate method: each scalar parameter is rescaled by its own running second moment, so a weight matrix is treated as a collection of scalars whose structure the update never sees. Structure-aware optimisers argue instead that transformer weights should be preconditioned as matrices. This section asks whether the most
prominent of them, Muon, carries its pretraining advantage into SFT, where the update is small
and, for LoRA, confined to a low-rank subspace.

Shampoo maintains Kronecker-factored preconditioners per tensor axis
\citep{gupta2018shampoo}; SOAP runs Adam inside Shampoo's eigenbasis \citep{vyas2024soap}; Muon
\citep{jordan2024muon} orthogonalises each 2D update by Newton--Schulz iteration, which is
steepest descent under the spectral norm \citep{bernstein2024modular}
(Appendix~\ref{app:optimizers} makes the equivalence precise). Whether these methods are equivalent was
debated in public in June 2026, in the Jordan--Anil ``is Muon just Shampoo'' exchange
\citep{jordan2026muonshampoo, anil2026muonshampoo}; the conclusion was that the abstract
preconditioner does not fix behaviour (momentum, accumulation, and implementation do), so we
document the implementation we use in Appendix~\ref{app:optimizers}. At pretraining scale the
evidence favours Muon: with weight decay and a matched update scale it trains
multi-billion-parameter models at roughly twice AdamW's compute efficiency
\citep{liu2025muonscalable}, expands the compute--time Pareto frontier \citep{essential2025muon},
and is used in frontier pretraining recipes such as the Kimi K2 line. Whether any of this survives
the small-update SFT regime is untested.

\paragraph{Design.} AdamW versus Muon at matched (model, dataset, global batch) on the same Megatron stack, selecting by validation NLL as in \S\ref{sec:foundation}. Muon orthogonalises each 2D weight update, so the natural test is full fine-tuning; we run Qwen3-8B on all four datasets at global batch $16$ for one epoch and leave the low-rank adapter case to future work. Each optimiser gets its own learning-rate grid, since Muon's update scale is not Adam's: the AdamW arm reuses the \S\ref{sec:foundation} FullFT grid ($3\times10^{-6}$ to $10^{-4}$) and the Muon arm sweeps $10^{-6}$ to $3\times10^{-4}$.

Muon's best validation NLL is at or below AdamW's on every dataset, by at most $0.012$ nats, and it reaches that minimum at a learning rate about three times lower (Figure~\ref{fig:optim-frontier}). The two arms therefore need separate grids: transferring AdamW's rate to Muon sits well past its optimum. Muon's frontier is also steeper: above its optimum the validation NLL climbs quickly while AdamW stays flat across the grid, so its loss advantage is narrow and available only when the rate is tuned down; at their best rates the two arms otherwise track near-identical loss trajectories (Figure~\ref{fig:optim-curves}).

At the loss-selected checkpoint, Muon reaches a flatter minimum on every dataset. Its per-parameter Fisher/Hessian trace is lower than AdamW's by a factor of $1.2$ to $2.8$ (Figure~\ref{fig:optim-flatness}; \S\ref{sec:downstream}, App.~\ref{app:hessian}, test split, $16\times4$ samples). At the fine-tuned scale loss ranked judged quality and the trace did not (\S\ref{sec:downstream}), so this is a property of the Muon solution, not usable for selection.

\begin{figure}[htbp]
  \centering
  \includegraphics[width=\linewidth]{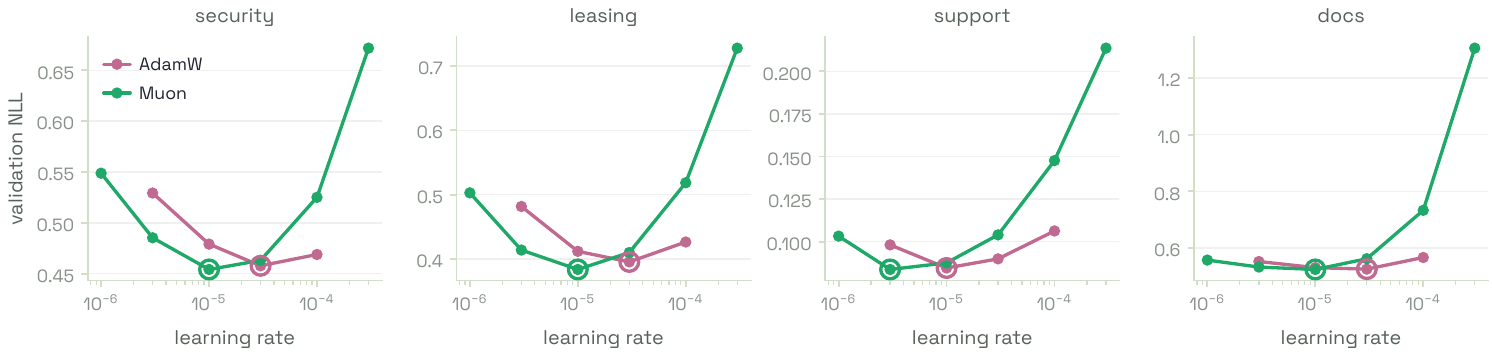}
  \caption{\textbf{Muon's loss optimum sits about $3\times$ below AdamW's and its frontier is far
  steeper.} Validation NLL against learning rate per dataset, AdamW (rose) and Muon (green), with
  each arm's best cell ringed. Muon is best at the low end and degrades sharply above its optimum;
  AdamW is flatter across the grid.}
  \label{fig:optim-frontier}
\end{figure}

\begin{figure}[htbp]
  \centering
  \includegraphics[width=0.6\linewidth]{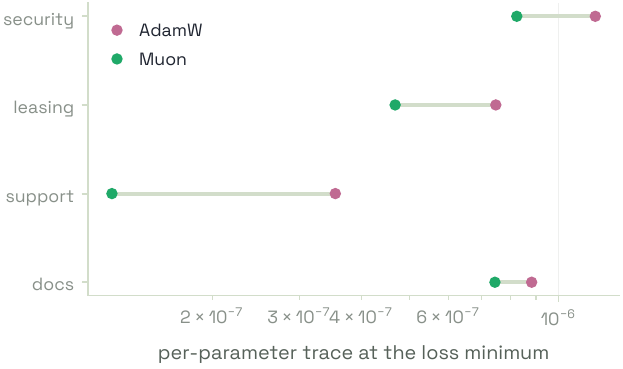}
  \caption{\textbf{Muon's loss minimum is flatter than AdamW's on every dataset.} Per-parameter
  Fisher/Hessian trace at the loss-selected checkpoint (test split, $16\times4$ samples, log axis),
  AdamW versus Muon; Muon is lower by a factor of $1.2$ to $2.8$.}
  \label{fig:optim-flatness}
\end{figure}

Deploying each arm's validation-NLL-selected checkpoint on the in-domain production judge ($200$ canonical-test rows) and on IFEval separates task quality from general capability (Figure~\ref{fig:optim-eval}). On the task judge the two arms are level, a mean difference of $0.015$ within the judge's resolution at $200$ rows, so the lower loss and flatter minimum do not improve judged task quality. On IFEval Muon scores higher on all four datasets, by a mean of $0.09$: the arm that reaches the flatter minimum retains more general instruction-following at matched task training, the consequence the flatness probe of \S\ref{sec:downstream} anticipated. This is one checkpoint per arm, so the task-judge result is read as no large effect rather than exact parity.

\begin{figure}[htbp]
  \centering
  \includegraphics[width=\linewidth]{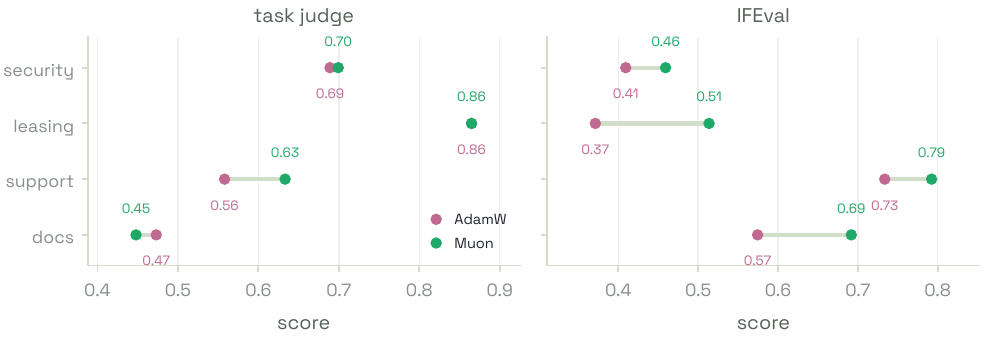}
  \caption{\textbf{Muon's lower loss and flatter minimum do not improve judged task quality, but Muon retains
  more general instruction-following.} In-domain production-judge score ($200$ canonical-test rows,
  left) and IFEval strict accuracy (right) for each optimiser's validation-NLL-selected Qwen3-8B
  FullFT checkpoint, AdamW (rose) versus Muon (green). On the task judge the two are level; on
  IFEval Muon is higher on every dataset.}
  \label{fig:optim-eval}
\end{figure}

%% file: sections/epochs.tex
\section{Epochs}
\label{sec:epochs}

Data volume is usually fixed by what exists, so the remaining axis is how many times to traverse
it. Two tensions set the optimum: more epochs reduce task loss while repeated tokens lose value
\citep{muennighoff2023dataconstrained}, and at some point general instruction-following
starts to erode, a cost that matters whenever the task instructions might change after
deployment.

\paragraph{Design.} An epoch sweep at the \S\ref{sec:foundation}-calibrated learning rate and batch per (model, dataset), tracking validation NLL, the downstream judge, and IFEval strict prompt accuracy \citep{zhou2023ifeval}, with the data-constrained law of \citet{muennighoff2023dataconstrained} as the reference for how repeated tokens decay. Each of $1,2,4,8$ epochs is a separate run with the full warmup-and-decay schedule, and the trained task is scored alongside the other three to read cross-task transfer. The grid is the two 8B models for LoRA and full fine-tuning, with a Qwen3-1.7B arm below 8B; a larger-data LoRA follow-up on the two 8B models (leasing, docs, security) separates fresh examples from repeated passes.

One split is between the metric we select on and the one we care about (Figure~\ref{fig:epochs-divergence}). At $5{,}000$ examples, validation NLL reaches its minimum by about two epochs and then rises monotonically to eight in every cell we ran (both 8B families, both tuners, all four tasks), by $40$--$120\%$ on the harder tasks, and harder still on the smaller Qwen3-1.7B. The judged task score over those same checkpoints does not follow the loss up: it holds or rises through eight epochs on both arms. Selecting epochs by validation loss would stop around two, while the production judge rates the eight-epoch checkpoint as good or better. This restates the \S\ref{sec:downstream} divergence on the epoch axis and cautions against early-stopping on validation loss alone.

\begin{figure}[htbp]
  \centering
  \includegraphics[width=\linewidth]{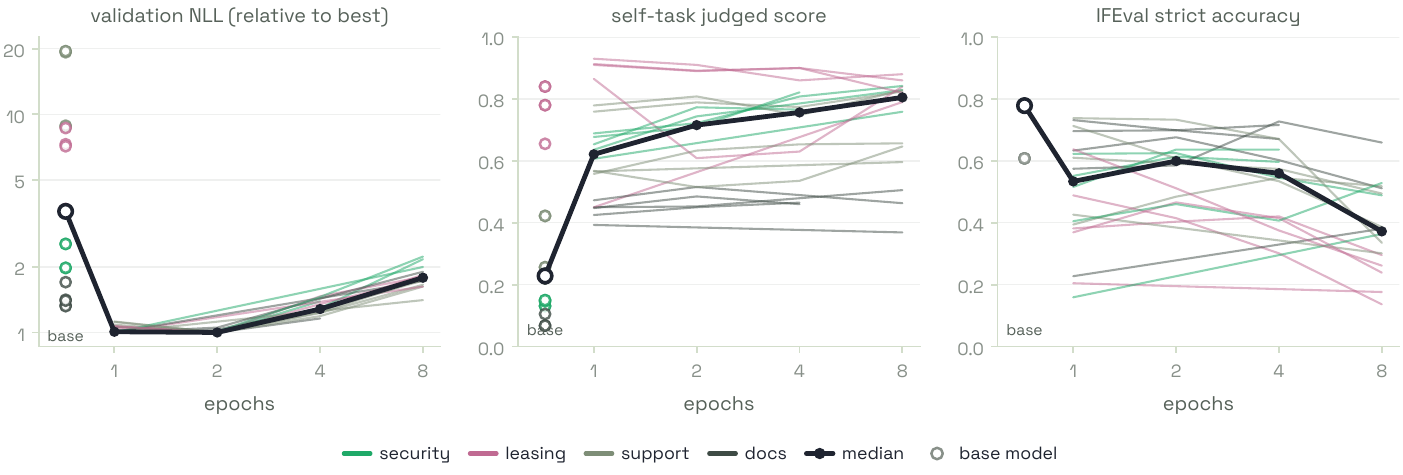}
  \caption{\textbf{With epochs the validation loss overfits and IFEval erodes, while the judged task
  score holds.} One thin line per (model, arm, task) cell, coloured by task, with the across-cell
  median in black and the untuned base model as open markers at the left. Left:
  validation NLL on a log axis, relative to each cell's best, so the fine-tuned minimum is $1$; the
  base sits several times higher, fine-tuning drops it by two epochs, then overfitting lifts it back.
  Middle: self-task judged score, which the base scores low and fine-tuning lifts then holds. Right:
  IFEval strict accuracy, which starts high at the base and erodes with epochs below it.}
  \label{fig:epochs-divergence}
\end{figure}

\subsection{Data volume: fresh examples over more passes}
\label{subsec:epochs-data-volume}

The fixed-$5\,000$ sweep leaves a practical ambiguity: if overfitting appears after two epochs, is
the right response to stop earlier or to add data and train longer? We test that directly on the
cells where the source datasets have enough filtered examples. The follow-up keeps the calibrated
LoRA learning rate and global batch for each (model, task), uses the two 8B models, drops the
support dataset, and crosses $10\,000$ and $15\,000$ examples with $1,2,4$ epochs. The
$5\,000$-example calibrated LoRA cells are loss anchors, so the loss figure spans $5$k,
$10$k, and $15$k examples. The downstream probes are lighter weight: each new
checkpoint is scored on 50 trained-task examples and 50 IFEval prompts. Cross-task transfer is
kept as an audited scratch artifact rather than a report result because several out-of-domain
generation files were dominated by endpoint timeouts. \S\ref{subsec:data-volume-law} reuses the
8B rows as anchors, while the scaling-law analysis extends the one-epoch slice across Qwen3-4B
and Qwen3-14B.

Fresh examples lower the loss; repeated passes do not. At matched example-pass budgets, replacing repeated $5{,}000$-example training with fresh $10{,}000$ examples lowers validation NLL in all $12$ comparisons (mean $16\%$), and $10{,}000$ to $15{,}000$ lowers it in all $18$ (mean $5.9\%$; Figure~\ref{fig:epochs-data-volume-loss}). The epoch optimum barely moves: within each example count the four-epoch run is worse than the best one- or two-epoch run in all $12$ larger-data curves, and the best NLL is reached at $15{,}000$ examples and two epochs in five of six cells. The rule is to add fresh examples and keep the epoch count near two, rather than train longer once data is larger.

\begin{figure}[htbp]
  \centering
  \includegraphics[width=\linewidth]{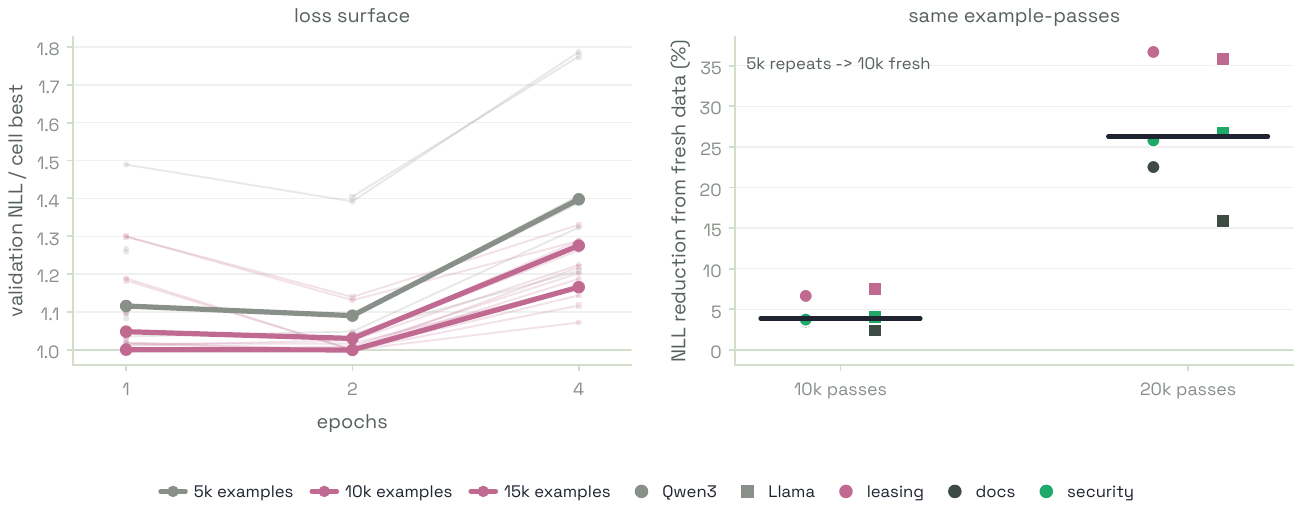}
  \caption{\textbf{More data lowers the loss curve, but it does not make four epochs optimal.}
  Left: validation NLL, normalised by each (model, task) cell's best value across the
  $5$k/$10$k/$15$k example grid. Thin lines are individual cells; thick lines are medians by
  example count. Right: matched example-pass comparisons where repeated $5$k training is replaced
  by fresh $10$k data. Every point is positive: fresh examples lower NLL at the same pass budget.}
  \label{fig:epochs-data-volume-loss}
\end{figure}

The larger-data downstream probes do not justify longer training (Figure~\ref{fig:epochs-data-volume-eval}). The 50-sample task scores are read directionally, and their later peaks come with no loss justification and with instruction-following erosion: IFEval falls from one to four epochs in ten of twelve larger-data cells, by a mean of $0.15$ and most on the instruction-heavy tasks. This agrees with the fixed-$5$k sweep: extra passes do not reliably improve the task and spend general instruction-following that more examples and fewer repeats would have kept.

\begin{figure}[htbp]
  \centering
  \includegraphics[width=0.82\linewidth]{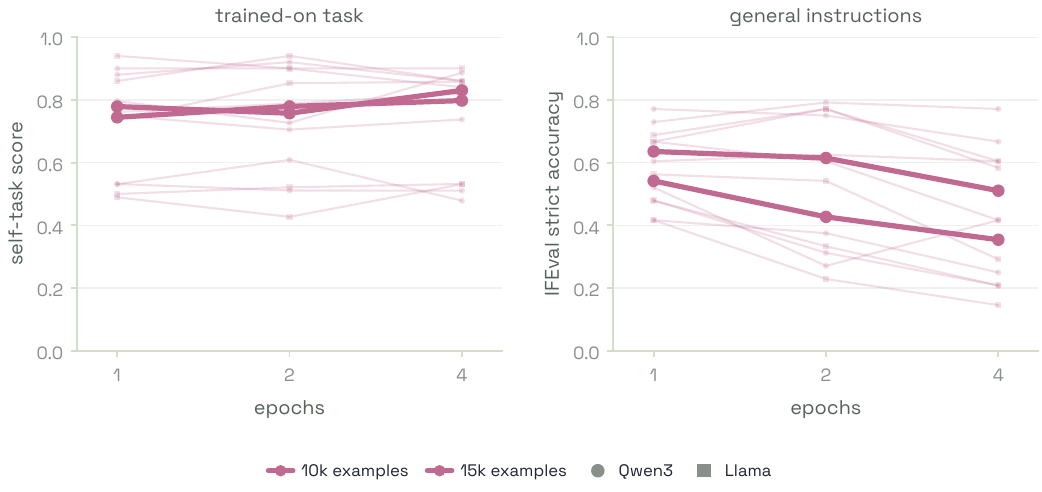}
  \caption{\textbf{The larger-data downstream probes do not justify longer training.} Each thin
  line is one (model, task, example count) cell from the $10$k/$15$k follow-up; thick lines are
  medians by example count. The trained-task score is noisy at 50 examples and sometimes peaks at
  four epochs. IFEval is cleaner: general instruction-following usually falls as epochs increase.}
  \label{fig:epochs-data-volume-eval}
\end{figure}

Across the full grid, the cost that matters is capability erosion. IFEval strict accuracy falls with epochs in about half the cells and steeply on the instruction-heavy tasks (Qwen3-8B full fine-tuning on support drops from $0.73$ at two epochs to $0.34$ at eight), while the extraction tasks, where base IFEval is already low, stay flat; the adapter erodes the same as full fine-tuning. Past about two epochs the validation loss overfits, the judged task score does not improve, and general instruction-following degrades. The useful number of epochs is therefore set by capability erosion; task quality and loss do not bind it. Most of the erosion comes from fine-tuning at all: against the untuned base (IFEval $0.78$ at 8B), even two epochs is costly, and the extra epochs past the loss optimum only deepen it.

At matched (model, task, epoch) the low-rank adapter matches or beats full fine-tuning on the task (Figure~\ref{fig:epochs-lora-vs-fullft}): it wins clearly on leasing and sits within noise on support, at far fewer parameters and no task penalty. The adapter cannot avoid the instruction-following cost: both arms erode IFEval with epochs by similar amounts, so the erosion is a property of training longer rather than of the tuning method.

Measured as the change from the untuned base on the same judge (Figure~\ref{fig:epochs-interference}, deepest epoch), fine-tuning lifts the trained task by $+0.32$ to $+0.66$ on docs, support, and security, and barely on leasing, where the base already scores well. Every off-diagonal entry is negative, by $-0.05$ to $-0.44$: training on one task lowers the base score on the others, steepest into the leasing judge. Fine-tuning specialises by raising its own task and lowering the rest, the general-capability cost the IFEval erosion also measures.

\begin{figure}[htbp]
  \centering
  \begin{minipage}[t]{0.53\linewidth}
    \centering
    \captionsetup{width=\linewidth}
    \includegraphics[width=\linewidth]{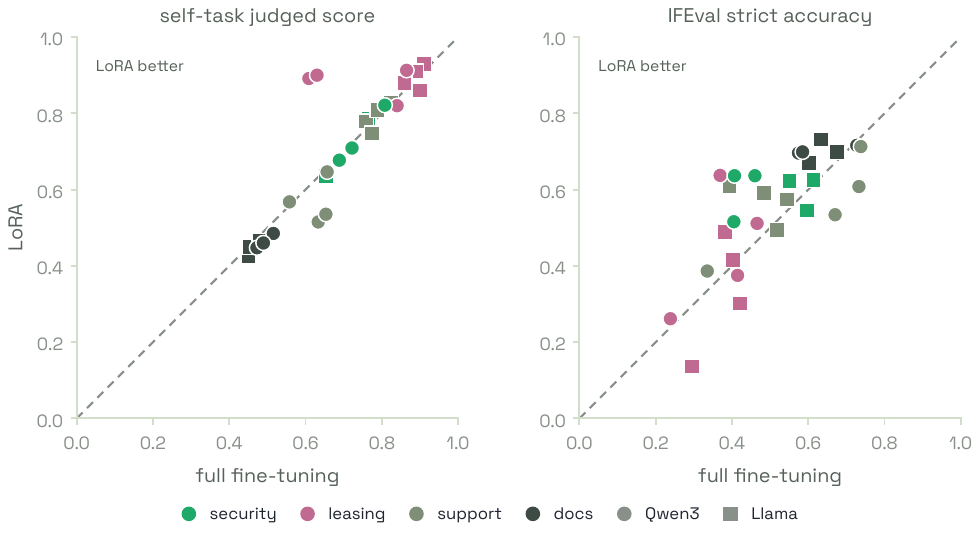}
    \caption{\textbf{LoRA is not a compromise on these tasks.} Each point is one (model, task, epoch)
    cell trained both ways: full fine-tuning on the x-axis, LoRA on the y-axis, so the diagonal is
    parity and points above it favour LoRA. Colour is the task, marker shape the model family. On the
    task score (left) LoRA sits on or above parity; on IFEval (right) the two arms scatter around it.}
    \label{fig:epochs-lora-vs-fullft}
  \end{minipage}\hfill
  \begin{minipage}[t]{0.45\linewidth}
    \centering
    \captionsetup{width=\linewidth}
    \includegraphics[width=\linewidth]{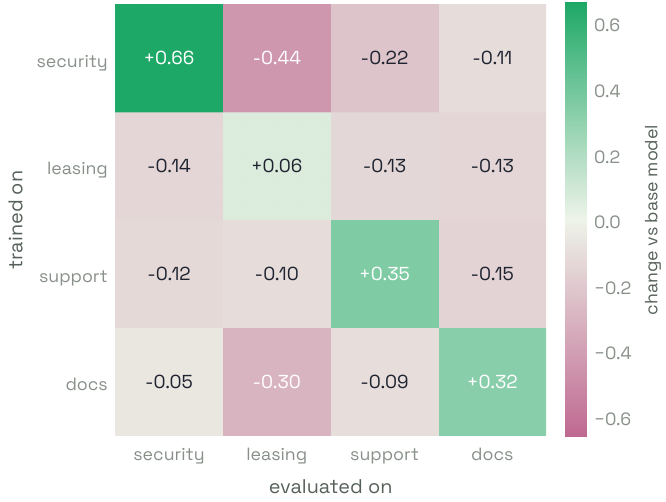}
    \caption{\textbf{Fine-tuning raises its own task and lowers the others, relative to the base
    model.} Each trained checkpoint (row) scored on every task judge (column) at the deepest epoch,
    averaged over the five (model, arm) configurations, as the change from the untuned base. Green is
    improvement, rose is degradation.}
    \label{fig:epochs-interference}
  \end{minipage}
\end{figure}

The scale grid spans Qwen3-1.7B and 8B. The 1.7B arm overfits more than the 8B models, not less: at the same $5{,}000$ examples it reaches a lower NLL early then climbs faster than the 8B models, consistent with a smaller model saturating the data sooner. We leave a 32B point to future work, as its end-of-training evaluation exceeds single-device memory at the full context.

The data-constrained law of \citet{muennighoff2023dataconstrained} describes how repeated tokens
lose value in the \emph{training} loss; here the \emph{validation} NLL overfits rather than decays,
so the relevant reading is the overfitting curve above rather than a fitted decay constant.

To separate the schedule from repetition, we trained one cell (Qwen3-8B, support) for eight epochs at a constant learning rate. It reaches essentially the same validation NLL as the decayed run ($0.13$ versus $0.14$) but erodes more (IFEval $0.15$ against $0.39$, and a lower task score). At a matched epoch budget the decay phase therefore recovers much of the erosion, so part of the cost of training longer is the sustained high learning rate rather than repeated tokens alone. This is a single probe, read as indicative.

%% file: sections/conclusion.tex
\section{Conclusion}
\label{sec:conclusion}

We aimed to remove per-job guesswork for post-training. We change one variable at a time, across two model families, dense and mixture-of-experts, LoRA and full fine-tuning, on four customer SFT datasets whose construction (Appendix~\ref{app:isft}) makes a
change in loss or judged score attributable to the lever rather than to the data.

The two levers every run must set behave predictably. The optimal LoRA learning rate is flat at $10^{-3}$ across $0.6$--$32$B and both families, roughly $33\times$ the full fine-tuning optimum, and global batch trades loss against cost and has no single best value (\S\ref{sec:foundation}). The rule transfers unchanged to a held-out 30B mixture-of-experts and keeps the monotone size trend when applied blind to 80B and 235B. The adapter adds capacity up to about rank $64$ and then plateaus, with $\alpha=32$ the best scale in every cell, so $r=64,\alpha=32$ stays the default and rank $32$ is the cheaper alternative (\S\ref{sec:lorahp}).

Validation loss, the metric every selection relies on, ranks judged quality within a fixed (model, dataset, recipe) cell (within-dataset-standardised Spearman $-0.38$ to $-0.88$), so selection by loss is sound inside a recipe (\S\ref{sec:downstream}). It does not transfer across model families, where a lower-NLL model can score worse with the judge, and the loss-landscape flatness proxy adds nothing reliable for in-domain selection at matched loss. Flatness instead tracks the general capability a fine-tune retains, most clearly across optimisers (\S\ref{sec:downstream}).

At the fixed defaults, loss falls as a saturating power law in size, steeper for full fine-tuning than for LoRA, so LoRA's $98\%$ retention narrows with scale, and all three MoEs enter the trend at roughly the geometric mean of their active and total parameters (\S\ref{sec:scaling}). Tripling fresh examples lowers one-epoch LoRA loss in every measured cell, with a return that is mostly dataset-specific. Muon carries a narrow version of its pretraining advantage into SFT, reaching a marginally lower validation NLL and a flatter minimum at a learning rate about $3\times$ below AdamW's, with task quality unchanged and more general instruction-following retained (\S\ref{sec:optimizers}). On the epoch axis the cost that matters is capability erosion: past about two epochs the validation loss overfits, the judged task score does not improve, and instruction-following degrades, while fresh examples beat repeated passes at a matched budget (\S\ref{sec:epochs}).

Each dataset is generated by iterative SFT to pass a customer-built evaluation (Appendix~\ref{app:isft}), so the supervised target carries little label noise and the downstream judge is a reliable, customer-validated criterion. The same construction is the main caveat: because the data is optimised to pass the judge, the judge is not independent of the training objective, so we read it as fit to the task as specified and use IFEval as a cross-criterion check.

Several axes remain open. Loss-based selection across model families inherits the fit-to-text-versus-behaviour gap of \S\ref{sec:downstream} and still needs a downstream check. The optimiser comparison covers full fine-tuning only, leaving open whether Muon's edge survives a rank-$64$ adapter (\S\ref{sec:optimizers}). The example-count results cannot separate examples from tokens within a dataset, since the two co-vary (\S\ref{subsec:data-volume-law}). The usable learning-rate ceiling is stability-bounded and model-dependent (\S\ref{subsec:instability}), so a widened grid must carry its own stability checks. And the epoch and scaling ladders stop where end-of-training evaluation exceeds single-device memory at full context, leaving a 32B epoch rung to future work. With four datasets and few sizes per family, we treat the directions of these results as the robust findings and quote magnitudes with their uncertainty.

Beyond the present grid, several directions remain. Reasoning models are one such direction: how one behaves when fine-tuned on data with no reasoning traces, or on off-policy traces from another model. The interaction with pretraining is a second, in particular whether continued pretraining on assistant outputs alone changes how a model later responds to SFT. On the method axis, the adapter study can widen beyond LoRA to learned weight deltas and to quantisation-aware training, with the deployment cost measured under post-training quantisation. The grid can also be pushed to newer model families and much longer data runs. And rejection-sampling fine-tuning, where the supervised set is the model's own grader-accepted outputs, connects the iSFT testbed here to on-policy RL.

%% file: sections/appendix_timing.tex
\section{Wall-clock, compute, and memory cost of the batch axis}
\label{app:timing}

\S\ref{sec:foundation} reads global batch as a loss/cost frontier; this appendix measures the cost
side. We run a controlled timing sweep over every Qwen3 cell (three models $\times$ four datasets
$\times$ global batch $\{8,16,32,64\}$, for both LoRA and FullFT) on a single 8$\times$H200 node
with in-training evaluation and checkpointing disabled. The reported cost is the steady-state
\emph{seconds per optimiser step} (the slope of Megatron's cumulative elapsed time over the
post-warmup iterations), which excludes setup, evaluation, and checkpoint writes. The metric is
reproducible to a coefficient of variation of $\approx 0.3\%$ across repeated seeds, whereas naive
end-to-end wall time over-reports true training time by $\approx 30\%$ (Megatron initialisation
plus the end-of-run evaluation). Throughput, peak memory, and the model dimensions used for FLOPs
are read from the same per-iteration logs; \texttt{micro\_batch\_size}\,$=\min(\text{global},32)$
throughout, matching the foundation study.

\paragraph{Per-step time scales with the batch; one-epoch time stays roughly fixed.}
Because the microbatch tracks the global batch up to 32, each optimiser step processes about
\texttt{global\_batch} examples, so seconds per step grows almost linearly with batch
(Figure~\ref{fig:timing-step}). At one fixed epoch the number of steps falls by the same factor,
so total one-epoch wall-clock is nearly batch-invariant: within about $10\%$ across the range,
and marginally lower at the largest batch as per-step overhead amortises
(Figure~\ref{fig:timing-epoch}). At a fixed token budget the cost of a small batch is its extra
optimiser steps; raising the batch adds almost no wall-clock until the memory limit is reached.

\paragraph{Compute efficiency.}
LoRA and FullFT cost almost the same seconds per step at matched settings, but FullFT performs
roughly $1.5\times$ the useful FLOPs per token ($6N$ versus $4N$, since LoRA skips the frozen base
weight gradients), so its achieved FLOP/s is correspondingly higher (Figure~\ref{fig:timing-flops}).
ms-swift's Megatron backend emits no FLOP counter, so we evaluate Megatron's per-token formula on
the logged GQA/SwiGLU dimensions (model FLOPs, no recompute). Achieved utilisation is low (single
digits to the low teens of the H200 bf16 peak) because the recipe pins tensor-parallel size to 8
even on the 0.6B model and long-sequence attention adds non-matmul work; a smaller tensor-parallel
degree on the small models would raise it.

\paragraph{Memory.}
Peak device memory rises while the microbatch grows with the batch, then flattens at global batch
64, where the extra batch comes from gradient accumulation rather than a larger microbatch
(Figure~\ref{fig:timing-mem}). FullFT sits well above LoRA at every size. Peak memory is what
limits how far the global batch can be raised.

\begin{figure}[htbp]
  \centering
  \includegraphics[width=\linewidth]{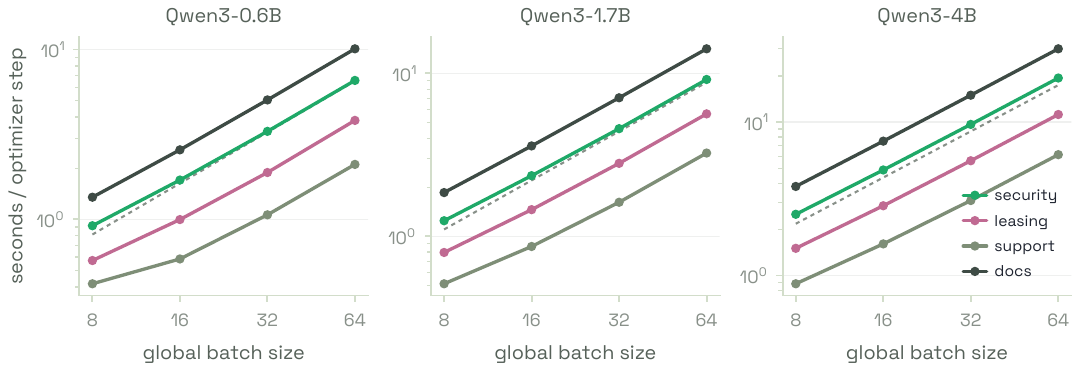}
  \caption{\textbf{Seconds per optimiser step grows almost linearly with global batch.}
  Steady-state per-step time (LoRA), faceted by model, one line per dataset; the dashed guide is
  exact linear scaling. Throughput is strongly dataset-bound through sequence length
  (\textsf{support} shortest, \textsf{docs} longest).}
  \label{fig:timing-step}
\end{figure}

\begin{figure}[htbp]
  \centering
  \includegraphics[width=\linewidth]{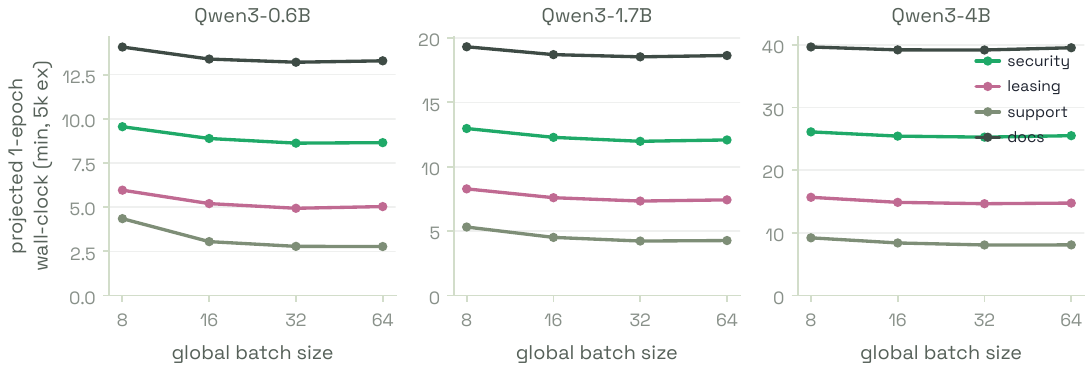}
  \caption{\textbf{One-epoch wall-clock is nearly batch-invariant.} Projected epoch time at the
  foundation $5{,}000$-example budget (LoRA), faceted by model, one line per dataset; the extra
  optimiser steps of a small batch cost almost no additional wall time, so batch should be chosen
  on loss and memory.}
  \label{fig:timing-epoch}
\end{figure}

\begin{figure}[htbp]
  \centering
  \includegraphics[width=0.7\linewidth]{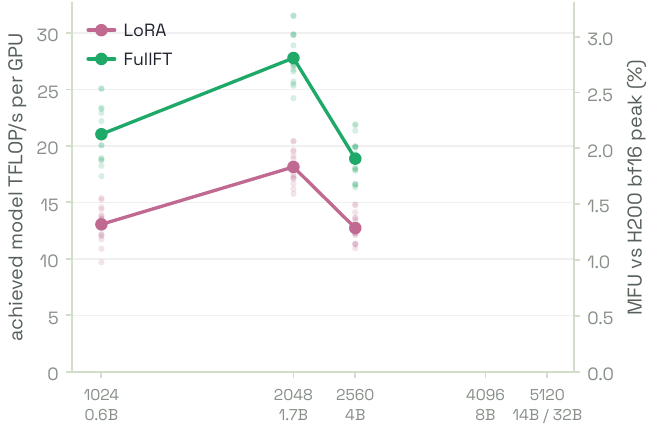}
  \caption{\textbf{FullFT achieves higher FLOP/s than LoRA by about the $6N/4N$ ratio; both sit
  far below peak.} Achieved model FLOP/s per GPU against model size (mean over datasets and batch
  sizes; faint points are the per-cell spread). The right axis expresses the same quantity as
  model-FLOPs utilisation against the H200 bf16 peak.}
  \label{fig:timing-flops}
\end{figure}

\begin{figure}[htbp]
  \centering
  \includegraphics[width=\linewidth]{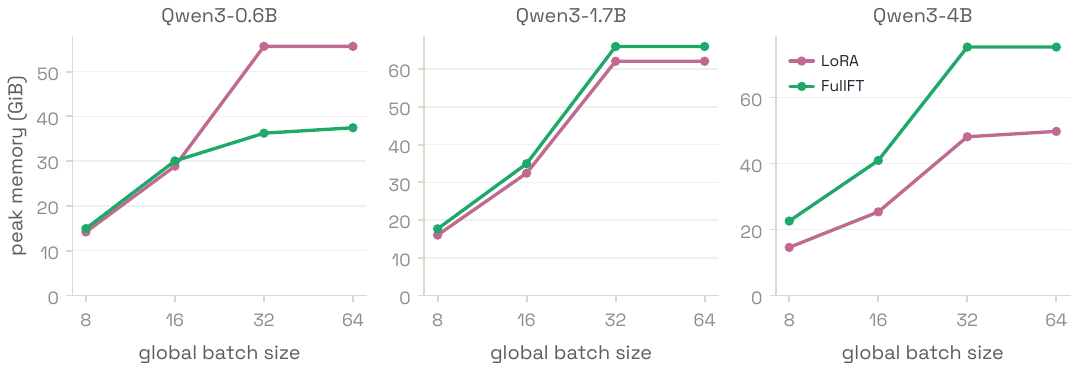}
  \caption{\textbf{Peak memory caps the batch.} Peak device memory against global
  batch size, faceted by model, for LoRA and FullFT. Memory grows with the microbatch (to 32) and
  flattens at batch 64, where gradient accumulation supplies the additional batch; FullFT is
  uniformly heavier.}
  \label{fig:timing-mem}
\end{figure}

%% file: sections/appendix_frontiers.tex
\section{Dataset statistics and coverage}
\label{app:data}

The four anonymised SFT datasets and their per-family token composition and eligible-cell counts.
Token counts are per example and differ across families because the tokenizers differ; ``trained
tokens'' is the total assistant-loss tokens over one $5{,}000$-example epoch. The datasets span an
order of magnitude in token load, which sets the achievable-loss floor of \S\ref{subsec:data}.

Of the $1{,}008$ planned dense cells, $970$ are eligible ($682$ LoRA, $288$ FullFT; $658$ Qwen3,
$312$ Llama). The $38$ excluded cells all fail to beat their pretrained baseline; every one is a
LoRA run at the top learning rate $3\!\times\!10^{-3}$, concentrated on the largest models
(Llama-3.1-8B contributes $16$); this is the high-learning-rate instability of
\S\ref{subsec:instability}. The MoE sweep (\S\ref{subsec:moe}) adds $112$ cells ($80$ LoRA, $32$
FullFT reference), of which $111$ are eligible; the one exclusion is again a LoRA run at
$3\!\times\!10^{-3}$ (batch $8$). The Qwen3-Next-80B blind test adds $8$ single-point cells
($4$ LoRA, $4$ FullFT), all eligible.

\begin{table}[htbp]
  \footnotesize
  \centering
  \caption{Per-dataset token composition (per example) and eligible-cell counts, by family.}
  \label{tab:data}
  \begin{tabular}{llrrrcc}
    \toprule
    Dataset & Family & Input tok & Asst.\ tok & Trained tok (M) & LoRA cells & FullFT cells \\
    \midrule
    \textsf{support}  & Qwen3 & $5{,}513$  & $167$   & $0.83$  & $116$ & $48$ \\
                      & Llama & $4{,}026$  & $129$   & $0.65$  & $55$  & $24$ \\
    \textsf{leasing}  & Qwen3 & $8{,}974$  & $219$   & $1.10$  & $117$ & $48$ \\
                      & Llama & $8{,}606$  & $192$   & $0.96$  & $55$  & $24$ \\
    \textsf{security} & Qwen3 & $12{,}959$ & $998$   & $4.99$  & $118$ & $48$ \\
                      & Llama & $12{,}808$ & $1{,}003$ & $5.01$ & $55$ & $24$ \\
    \textsf{docs}     & Qwen3 & $20{,}084$ & $3{,}019$ & $15.09$ & $115$ & $48$ \\
                      & Llama & $19{,}911$ & $2{,}967$ & $14.83$ & $51$  & $24$ \\
    \midrule
    \textsf{support}  & MoE   & $5{,}494$  & $147$   & $0.73$  & $19$  & $8$ \\
    \textsf{leasing}  & MoE   & $8{,}952$  & $198$   & $0.99$  & $20$  & $8$ \\
    \textsf{security} & MoE   & $12{,}955$ & $994$   & $4.97$  & $20$  & $8$ \\
    \textsf{docs}     & MoE   & $20{,}080$ & $3{,}015$ & $15.07$ & $20$ & $8$ \\
    \bottomrule
  \end{tabular}
\end{table}

\section{The iterative-SFT testbed}
\label{app:isft}

Each of the four datasets is produced by the iterative supervised fine-tuning (iSFT) pipeline
\citep{oneill2025isft}, and that construction makes the loss and judge comparisons in this
report well-posed. This appendix records the loop and the one limit it places on the downstream
readout.

Each task begins with an evaluation, built with the customer's domain experts over weeks, that scores a candidate output against the task's acceptance criteria (a graded rubric for the writing tasks, a pass/fail or micro-F1 check for the extraction tasks; \S\ref{subsec:base-calibration} lists the four). Training examples are then generated rather than hand-written: a model drafts an output, the evaluator grades it and returns structured feedback on why it fell short, and the model revises until it passes. Only passing outputs are kept, so a run trains on inputs paired with outputs that already score well under the task's own evaluation, with no human-written gold answer in the loop.

Two properties follow, both of which matter for a one-lever-at-a-time sweep. The supervised target is internally consistent: because every output was accepted by the same evaluator, the dataset carries little of the label noise or contradictory gold answers that make absolute loss hard to interpret, so a change in validation NLL between two cells reflects the lever under study rather than disagreement in the labels. And the evaluation is reliable enough to be the downstream judge: the production judge of \S\ref{sec:downstream} is the same customer-validated criterion the data was generated to satisfy, which is why it resolves base-model quality (\S\ref{subsec:base-calibration}) and ranks fine-tuned cells within a recipe (\S\ref{subsec:finetuned-grid}). Generating data to pass the grader also avoids the plain-SFT regime in which adding noisier examples eventually hurts, so the example-count results of \S\ref{subsec:data-volume-law} and \S\ref{subsec:epochs-data-volume} measure the return to more clean data.

The construction that makes the judge reliable also makes it non-independent: because the training data was optimised to pass the evaluation, the judge is the criterion the data targets rather than an oracle external to it. ``Clean'' here means a well-posed target with low label noise and a trustworthy, customer-validated readout; it does not mean the downstream score is independent of the training objective. We therefore read the judge as a measure of fit to the task as specified, and use the cross-criterion IFEval probe (\S\ref{sec:epochs}) to check that a fine-tune has not gained task fit by eroding general capability the evaluation does not score.

\section{LoRA retention by model}
\label{app:lora-retention}

The per-model breakdown behind the median $98\%$ retention of \S\ref{sec:foundation}: the fraction
of full fine-tuning's improvement over the pretrained baseline that LoRA recovers, and the LoRA
trainable-parameter percentage, for each model.

\begin{table}[htbp]
  \footnotesize
  \centering
  \caption{\textbf{LoRA recovers a median $98\%$ of FullFT's gain at $3$--$13\%$ of the
  parameters.} ``Retained'' is the median fraction of FullFT's improvement over the pretrained
  baseline that LoRA recovers, over matched (dataset, batch) cells; ``trainable'' is the LoRA
  trainable-parameter percentage.}
  \label{tab:lora}
  \begin{tabular}{lcc@{\hskip 2.5em}lcc}
    \toprule
    Model & Retained & Trainable & Model & Retained & Trainable \\
    \midrule
    Qwen3-0.6B & $97\%$ & $12.6\%$ & Qwen3-14B & $99\%$ & $3.8\%$ \\
    Qwen3-1.7B & $99\%$ & $8.6\%$  & Qwen3-32B & $99\%$ & $3.1\%$ \\
    Qwen3-4B   & $99\%$ & $6.5\%$  & Llama-3.2-1B & $97\%$ & $7.3\%$ \\
    Qwen3-8B   & $98\%$ & $4.8\%$  & Llama-3.2-3B & $97\%$ & $6.8\%$ \\
               &        &          & Llama-3.1-8B & $99\%$ & $4.5\%$ \\
    \bottomrule
  \end{tabular}
\end{table}

\section{Size-law fit parameters}
\label{app:sizelaw}

The per-dataset saturating-fit parameters behind the size trend of \S\ref{subsec:sizelaw}.

\begin{table}[htbp]
  \footnotesize
  \centering
  \caption{\textbf{Saturating size-trend fits per dataset at the fixed defaults.} $L_\infty$ and
  $\alpha$ with pairs-bootstrap $95\%$ intervals ($B=1{,}000$), the exponent re-fit on the
  oracle learning rate, and the fit RMSE; six Qwen3 sizes per fit, total parameters in
  billions.}
  \label{tab:sizelaw}
  \begin{tabular}{llcccc}
    \toprule
    Tuner & Dataset & $L_\infty$ & $\alpha$ & $\alpha$ (LR oracle) & RMSE \\
    \midrule
    LoRA & \textsf{security} & $0.376$ \;{\footnotesize$[0.24,0.42]$} & $0.32$ \;{\footnotesize$[0.16,0.64]$} & $0.32$ & $0.0039$ \\
         & \textsf{leasing}  & $0.311$ \;{\footnotesize$[0.12,0.37]$} & $0.31$ \;{\footnotesize$[0.14,0.70]$} & $0.31$ & $0.0048$ \\
         & \textsf{support}  & $0.069$ \;{\footnotesize$[0.05,0.08]$} & $0.33$ \;{\footnotesize$[0.20,0.56]$} & $0.33$ & $0.0007$ \\
         & \textsf{docs}     & $0.349$ \;{\footnotesize$[0.30,0.41]$} & $0.40$ \;{\footnotesize$[0.35,0.61]$} & $0.40$ & $0.0039$ \\
    \midrule
    FullFT & \textsf{security} & $0.388$ \;{\footnotesize$[0.29,0.43]$} & $0.40$ \;{\footnotesize$[0.21,0.82]$} & $0.34$ & $0.0049$ \\
           & \textsf{leasing}  & $0.326$ \;{\footnotesize$[0.24,0.36]$} & $0.40$ \;{\footnotesize$[0.23,0.77]$} & $0.39$ & $0.0044$ \\
           & \textsf{support}  & $0.081$ \;{\footnotesize$[0.07,0.08]$} & $0.51$ \;{\footnotesize$[0.34,0.89]$} & $0.33$ & $0.0004$ \\
           & \textsf{docs}     & $0.358$ \;{\footnotesize$[0.29,0.42]$} & $0.41$ \;{\footnotesize$[0.34,0.66]$} & $0.41$ & $0.0041$ \\
    \bottomrule
  \end{tabular}
\end{table}

\section{Per-model learning-rate frontiers}
\label{app:frontiers}

The full learning-rate response curves underlying \S\ref{subsec:law}. For legibility the six Qwen
sizes are split into small (0.6--4B) and large (8--32B) grids; Llama's three fit one grid. Each
panel plots validation NLL against learning rate, one line per global batch, with LoRA solid and
the FullFT reference arm dashed on its own (lower) grid, the within-$0.01$-nat LoRA band shaded,
the recommended learning rates marked ($10^{-3}$ LoRA dashed, $3\!\times\!10^{-5}$ FullFT
dotted), and a $\times$ marking grid-edge optima.

\begin{figure}[htbp]
  \centering
  \includegraphics[width=\linewidth]{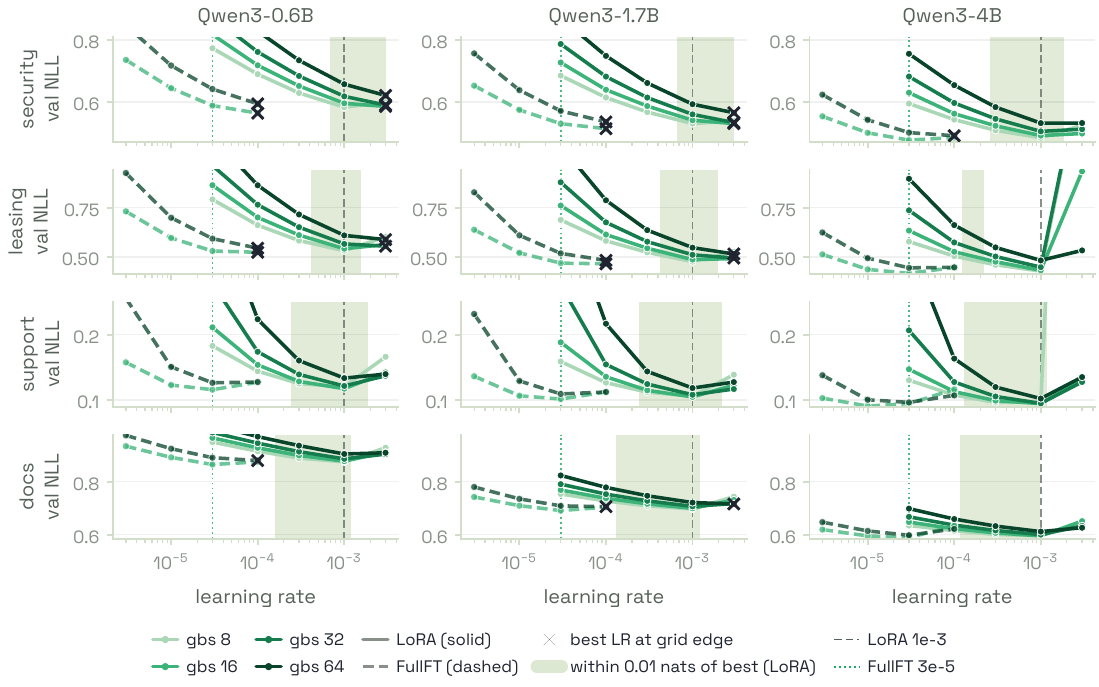}
  \caption{Learning-rate frontiers for the small Qwen3 sizes (0.6B, 1.7B, 4B), LoRA and FullFT,
  faceted by dataset (rows) and model (columns).}
  \label{fig:frontiers-qwen-small}
\end{figure}

\begin{figure}[htbp]
  \centering
  \includegraphics[width=\linewidth]{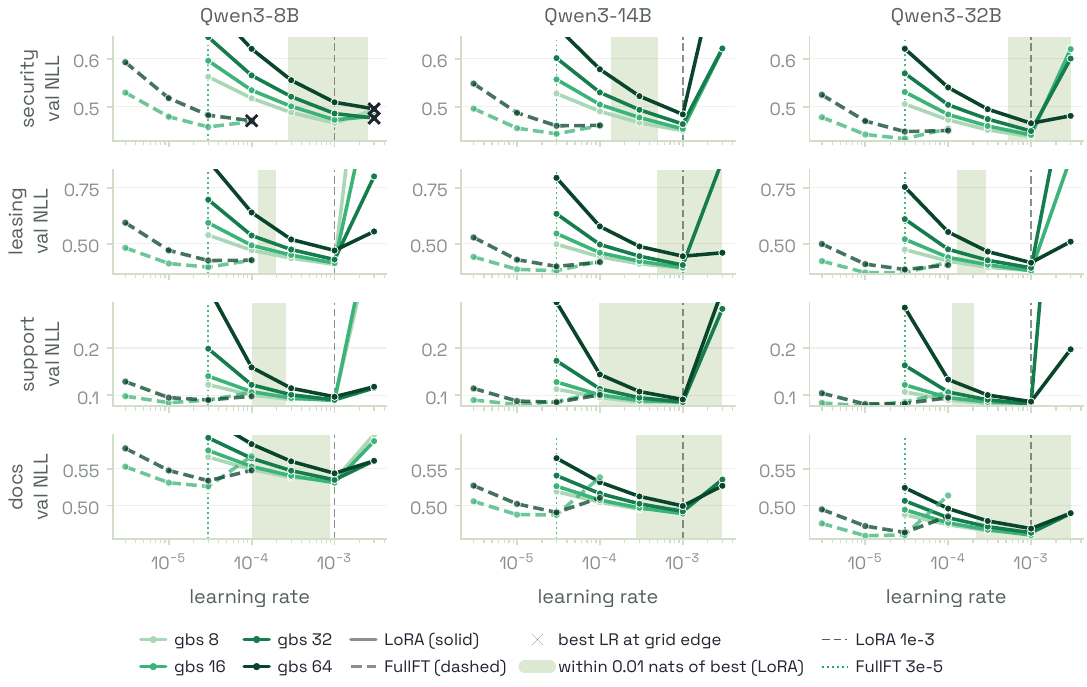}
  \caption{Learning-rate frontiers for the large Qwen3 sizes (8B, 14B, 32B), LoRA and FullFT,
  faceted by dataset (rows) and model (columns).}
  \label{fig:frontiers-qwen-large}
\end{figure}

\begin{figure}[htbp]
  \centering
  \includegraphics[width=\linewidth]{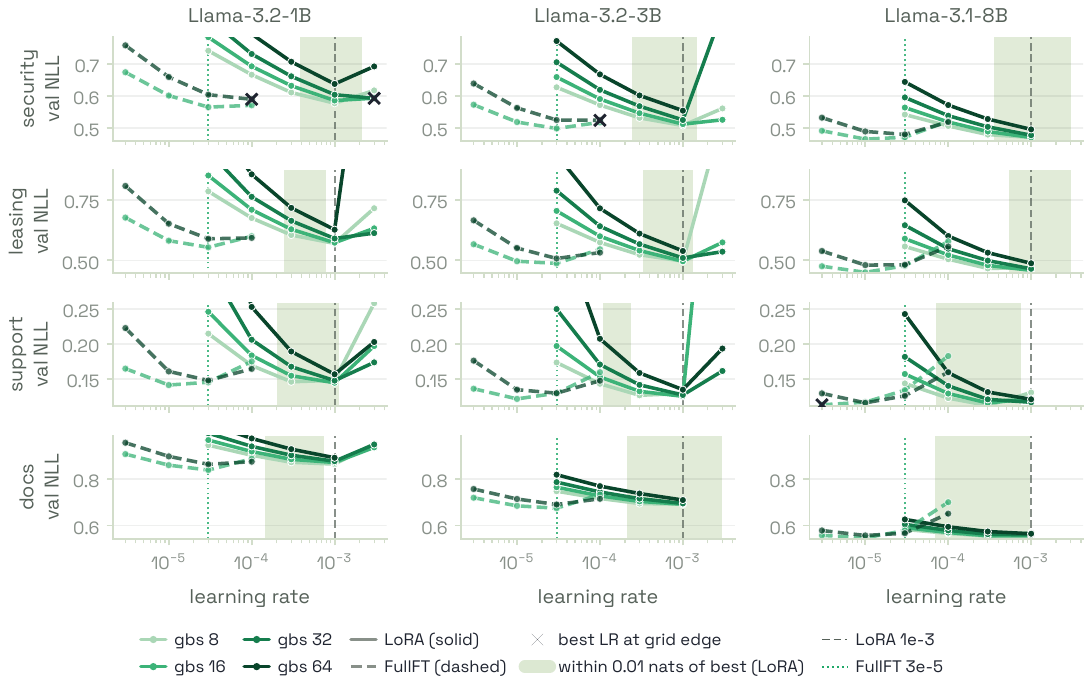}
  \caption{Learning-rate frontiers for the Llama family (3.2-1B, 3.2-3B, 3.1-8B), LoRA and
  FullFT, faceted by dataset (rows) and model (columns).}
  \label{fig:frontiers-llama}
\end{figure}

\section{Learning-rate schedule trajectories}
\label{app:lr-schedule-curves}

Figure~\ref{fig:lr-schedule} summarises the matched constant-vs-cosine schedule probe by
best-checkpoint loss and end-of-training rebound. The corresponding validation trajectories below
show the same effect directly: at $10^{-4}$ constant learning rate preserves a slightly larger step
size throughout training; at $10^{-3}$ cosine is the better calibrated default; and at
$3\!\times\!10^{-3}$ constant learning rate often reaches a plausible intermediate checkpoint and
then drifts upward late in the run.

\begin{figure}[htbp]
  \centering
  \includegraphics[width=\linewidth]{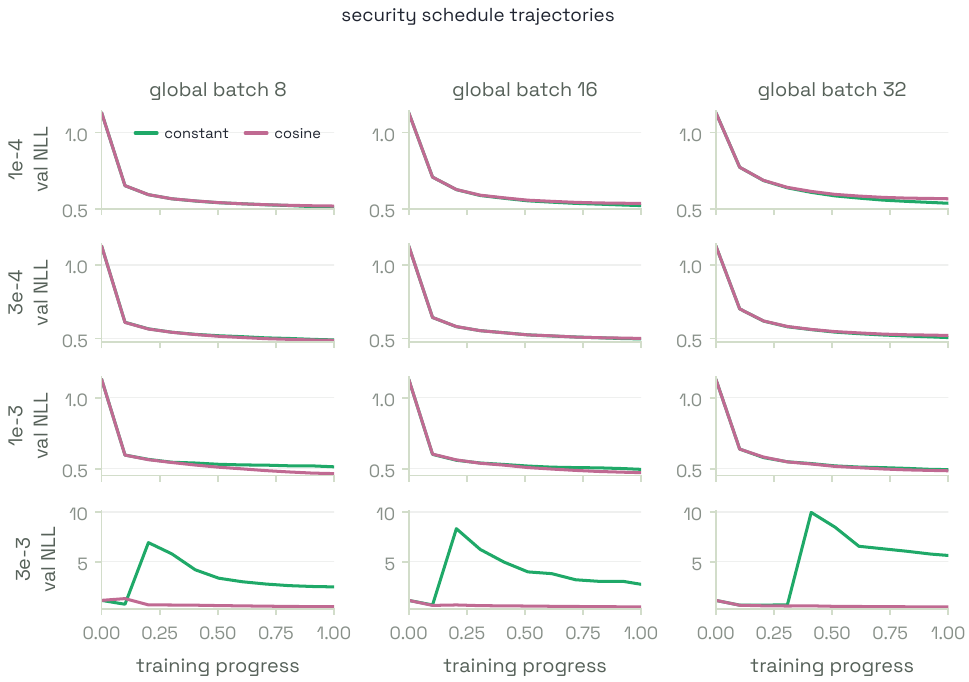}
  \caption{\textbf{Security validation trajectories for the schedule probe.} Rows are learning
  rates, columns are global batch sizes, and each panel overlays the constant and cosine schedules
  for the matched Qwen3-8B LoRA cell. The high-learning-rate constant runs rebound sharply late in
  training, which is the mechanism behind the end-gap summary in Figure~\ref{fig:lr-schedule}.}
  \label{fig:lr-schedule-curves-security}
\end{figure}

\begin{figure}[htbp]
  \centering
  \includegraphics[width=\linewidth]{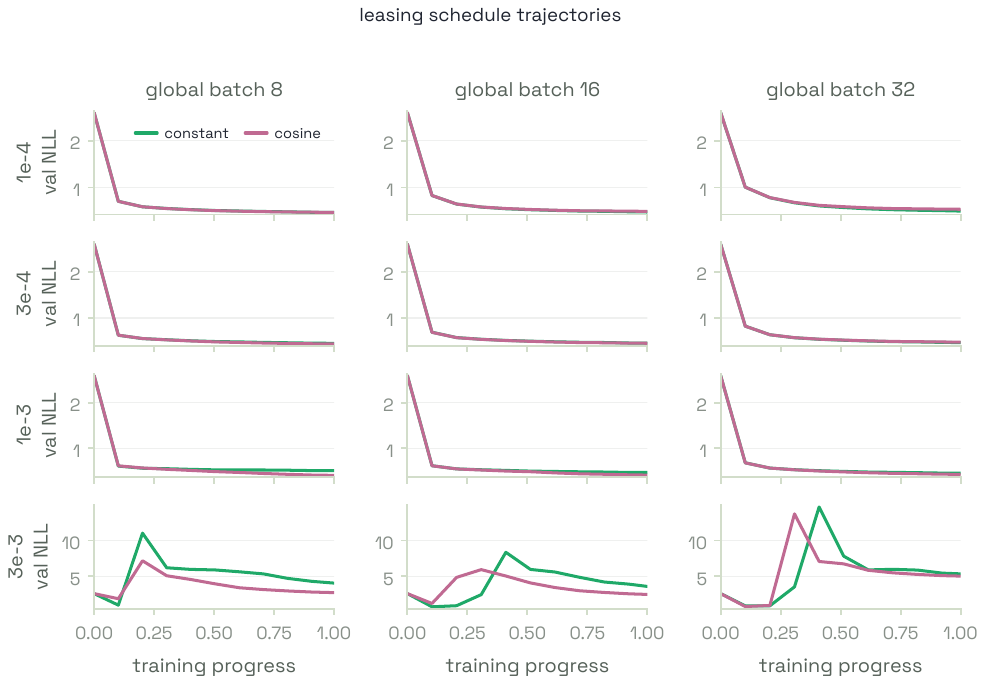}
  \caption{\textbf{Leasing validation trajectories for the schedule probe.} The same matched
  grid shows why constant learning rate is not a better default: the low-LR cells improve slightly,
  but the calibrated $10^{-3}$ cells favour cosine and the $3\!\times\!10^{-3}$ cells expose
  late-run instability.}
  \label{fig:lr-schedule-curves-leasing}
\end{figure}

%% file: sections/appendix_hessian.tex
\section{Estimating the Hessian trace: a first-order identity}
\label{app:hessian}

\S\ref{sec:downstream} uses the flatness of the converged SFT loss, the trace of its Hessian, as a
second signal alongside validation loss. Computing $\operatorname{Tr}\nabla^2_\theta L$ directly is
infeasible: the Hessian is $d\times d$ with $d\sim 10^{8}$ trainable parameters and is never
formed. We instead use an identity that replaces every second derivative with an average of
squared gradient norms. The argument has three steps. Differentiating the normalisation constraint
(predicted probabilities sum to one) twice gives two zero-sum identities,
\eqref{eq:norm-identities}. The Hessian of a log-probability splits into a curvature term and an
outer-product term, \eqref{eq:logp-split}. Averaging the curvature term over tokens drawn from the
\emph{model} makes it vanish by the first identity, leaving only the outer product
(Lemma~\ref{lem:bartlett}); taking traces then turns that outer product into a squared gradient
norm (Corollary~\ref{cor:trace}). The only assumption is that the model is well fit, and
Remark~\ref{rem:regime} states what the estimator measures when it is not.

\paragraph{Setup.} Reduce each masked position to next-token classification over a vocabulary of
$c$ classes. A context $x_{-t}$ (a position $t$ and its prefix) is drawn from the data
distribution $\mathcal{D}$; the model maps it to a predicted distribution
$f_\theta(x_{-t})\in\Delta^{c-1}$, with class probabilities $p_k(\theta)=[f_\theta(x_{-t})]_k$.
The per-position loss is the negative log-likelihood of the true token,
$\ell(\theta)=-\log p_{x_t}(\theta)$, and the full loss $L(\theta)$ averages $\ell$ over
$\mathcal{D}$. Because the predicted distribution is normalised for \emph{every} $\theta$,
differentiating $\sum_{k=1}^{c} p_k(\theta)=1$ once and twice gives
\begin{equation}
  \sum_{k=1}^{c}\nabla_\theta\, p_k(\theta)=0,
  \qquad
  \sum_{k=1}^{c}\nabla^2_\theta\, p_k(\theta)=0,
  \label{eq:norm-identities}
\end{equation}
since the gradient and Hessian of the constant $1$ vanish. For a fixed class $k$ the chain rule
gives $\nabla_\theta\log p_k = \nabla_\theta p_k / p_k$, and differentiating once more with the
quotient rule yields
\begin{equation}
  \nabla^2_\theta\log p_k
  \;=\;
  \frac{\nabla^2_\theta\, p_k}{p_k}
  \;-\;
  \nabla_\theta\log p_k\,(\nabla_\theta\log p_k)^\top .
  \label{eq:logp-split}
\end{equation}
The Hessian of a log-probability splits into a second-order ``curvature of the probability'' term
and a first-order outer-product term. The derivation removes the first term and keeps the second.

\begin{lemma}[Bartlett identity for the loss Hessian]
\label{lem:bartlett}
Assume the model is in the saturation regime at $\theta$: for $\mathcal{D}$-almost every context,
its predicted conditional equals the data conditional, $f_\theta(x_{-t})=\Pr(\cdot\mid x_{-t})$
(the global-minimiser condition of \citealp[Lemma 3.2]{liu2023implicitbias}). Then the loss
Hessian equals the Fisher information,
\begin{equation}
  \nabla^2_\theta L(\theta)
  \;=\;
  \mathbb{E}_{(t,\,x_{-t})\sim\mathcal{D}}\;
  \mathbb{E}_{x_t\sim f_\theta(x_{-t})}\!\left[
    \nabla_\theta\log[f_\theta(x_{-t})]_{x_t}\,
    \big(\nabla_\theta\log[f_\theta(x_{-t})]_{x_t}\big)^\top
  \right],
  \label{eq:bartlett}
\end{equation}
where the outer expectation is over data contexts and the inner expectation samples the token from
the model's own predicted distribution.
\end{lemma}

\begin{proof}
From $\ell=-\log p_{x_t}$ and \eqref{eq:logp-split},
$\nabla^2_\theta\ell = -\nabla^2_\theta p_{x_t}/p_{x_t} + \nabla\log p_{x_t}(\nabla\log p_{x_t})^\top$.
Average the first term over $x_t$ drawn from the model, for which $\Pr(x_t=k)=p_k$:
\begin{equation*}
  \mathbb{E}_{x_t}\!\left[\frac{\nabla^2_\theta\, p_{x_t}}{p_{x_t}}\right]
  \;=\;
  \sum_{k=1}^{c} p_k\,\frac{\nabla^2_\theta\, p_k}{p_k}
  \;=\;
  \sum_{k=1}^{c}\nabla^2_\theta\, p_k
  \;=\;0 ,
\end{equation*}
by \eqref{eq:norm-identities}: the sampling weight $p_k$ cancels the $1/p_k$, leaving the bare sum
of second derivatives, which is the Hessian of the constant $\sum_k p_k = 1$. The saturation
assumption is what licenses sampling $x_t$ from the model rather than from the data: under it the
two distributions agree, so this average over model samples is also the average over data, and the
cancellation applies to $L$. Without it the weights would not be $p_k$ and the term would not
vanish. Only the outer-product term survives, and taking the outer expectation over contexts gives
\eqref{eq:bartlett}. (The Fisher matrix differs from the gradient covariance by
$-\nabla L\,(\nabla L)^\top$, which vanishes at a minimiser $\nabla L=0$; hence Hessian, Fisher,
and gradient covariance coincide there \citep{bartlett1953, martens2020}.)
\end{proof}

\begin{corollary}[Trace estimator]
\label{cor:trace}
Because $\operatorname{Tr}(gg^\top)=\lVert g\rVert_2^2$ and the trace is linear, it passes through
the expectations in \eqref{eq:bartlett}:
\begin{equation}
  \operatorname{Tr}\nabla^2_\theta L(\theta)
  \;=\;
  \mathbb{E}_{(t,\,x_{-t})\sim\mathcal{D}}\;
  \mathbb{E}_{x_t\sim f_\theta(x_{-t})}\!\left[
    \big\lVert \nabla_\theta\log[f_\theta(x_{-t})]_{x_t}\big\rVert_2^2
  \right].
  \label{eq:trace}
\end{equation}
The total curvature equals the expected squared length of a single sampled-token log-probability
gradient, a quantity with no second derivatives.
\end{corollary}

\begin{remark}[What the trace measures]
\label{rem:meaning}
For $u$ uniform on the unit sphere in $\mathbb{R}^d$ and any symmetric $H$,
$\mathbb{E}_u[u^\top H u]=\operatorname{Tr}H/d$, so a second-order expansion of the loss along a
random direction gives
\begin{equation}
  \mathbb{E}_u\big[L(\theta+\varepsilon u)\big] - L(\theta)
  \;=\;
  \frac{\varepsilon^2}{2d}\,\operatorname{Tr}\nabla^2_\theta L(\theta) + O(\varepsilon^3)
  \qquad\text{at a minimiser.}
  \label{eq:directional}
\end{equation}
The per-parameter trace $\operatorname{Tr}\nabla^2 L/d$ reported in \S\ref{sec:downstream} is
therefore the mean curvature of the loss along a random direction: it measures how much the loss
rises under an isotropic perturbation of the trained parameters, so a small trace means a flat minimum.
\end{remark}

\paragraph{Estimator and implementation.} Equation \eqref{eq:trace} is a Monte-Carlo expectation
requiring, per sample, only: (i)~a masked context, (ii)~one token sampled from the model's
predicted distribution at that position, (iii)~one backward pass for the scalar $\log p_{x_t}$,
and (iv)~its squared gradient norm. There are no second derivatives (they were removed by
\eqref{eq:norm-identities}) and no per-class Jacobian, since the sampling in
Lemma~\ref{lem:bartlett} already integrated over classes. We draw $E$ examples and $P$
probe positions per example (every trace in this report uses $E=16$, $P=4$, i.e.\ $64$ samples, on
the held-out test split), choose each probe position deterministically among the assistant-label
positions (hashed by row, for reproducibility), sample the token under a fixed per-probe seed, and
accumulate $\lVert\nabla_\theta\log p_{x_t}\rVert_2^2$ from a single \texttt{autograd.grad} call
taken over the \emph{trainable} parameters; for LoRA this is the adapter subspace, so the reported
curvature is the subspace flatness whose pitfalls \S\ref{sec:downstream} discusses
\citep{flatlora2025}. The fine-tuned grid's trace was re-measured at this same $64$-sample
test-split setting to match the base-model anchor, superseding the campaign's original $4$-sample
validation-split measurement. This is far cheaper than matrix-free Hutchinson--Lanczos curvature
probes \citep{yao2020pyhessian}, which still require Hessian--vector products; here one first-order
gradient per sample suffices.

\begin{remark}[Regime caveat]
\label{rem:regime}
Strictly, \eqref{eq:trace} estimates the trace of the \emph{Fisher information}. It equals the
Hessian trace only in the saturation regime of Lemma~\ref{lem:bartlett}, a fully fit minimiser,
and is otherwise a flatness proxy rather than the exact curvature. We therefore report it as a
Fisher-trace flatness measure and read it alongside the loss.
\end{remark}

%% file: sections/appendix_optimizers.tex
\section{AdamW and Muon: update rules and implementation}
\label{app:optimizers}

This appendix states the two update rules precisely and documents the Muon implementation we use
(Megatron-Core via ms-swift), whose choices affect behaviour (\S\ref{sec:optimizers}).
The structure: Muon is built around one object, the matrix sign $\operatorname{msign}$, which has
three equivalent characterisations (Remark~\ref{rem:msign}); a Newton--Schulz iteration computes
it with matrix multiplies only, acting on the spectrum one singular value at a time
(\eqref{eq:ns}); and a pair of variational identities places AdamW and Muon in the same
steepest-descent frame under different norms, \eqref{eq:twin}, which also gives the relationship to Shampoo.

\paragraph{AdamW.} For a parameter $\theta$ with stochastic gradient $g_t$, AdamW
\citep{loshchilov2019adamw} maintains per-coordinate moments and steps by
\begin{equation}
  m_t = \beta_1 m_{t-1} + (1-\beta_1) g_t,\quad
  v_t = \beta_2 v_{t-1} + (1-\beta_2) g_t^2,\quad
  \theta_t = \theta_{t-1} - \eta\Big(\tfrac{\hat m_t}{\sqrt{\hat v_t}+\epsilon} + \lambda\,\theta_{t-1}\Big),
  \label{eq:adamw}
\end{equation}
with bias-corrected $\hat m_t,\hat v_t$ and decoupled weight decay $\lambda$. The preconditioner
$\operatorname{diag}(1/\sqrt{\hat v_t})$ is diagonal: every scalar is rescaled by its own running
gradient magnitude, independently of any matrix structure.

\paragraph{Muon.} For a 2D weight $W$ with gradient $G_t$, Muon \citep{jordan2024muon} runs
momentum SGD and then orthogonalises the resulting matrix before applying it:
\begin{equation}
  B_t = \mu B_{t-1} + G_t,\qquad
  O_t = \operatorname{msign}(B_t),\qquad
  W_t = W_{t-1} - \eta\, s\, O_t,
  \label{eq:muon}
\end{equation}
where $s$ is a scale factor (below) and $\operatorname{msign}$ is defined next.

\begin{remark}[Three views of $\operatorname{msign}$]
\label{rem:msign}
Let $B=U\Sigma V^\top$ be a singular value decomposition with $\Sigma$ invertible. The following
define the same matrix $\operatorname{msign}(B)=UV^\top$, which is $B$ with every singular value
set to $1$:
\begin{enumerate}
  \item \emph{Matrix sign.} $\operatorname{msign}(B) = B\,(B^\top B)^{-1/2}$, the matrix analogue
  of $\operatorname{sign}(x)=x/\lvert x\rvert$: the magnitude (spectrum) is divided out and the
  directions (singular vectors) are kept.
  \item \emph{Nearest semi-orthogonal matrix.}
  $\operatorname{msign}(B) = \arg\min_{O\,:\,O^\top O=I}\lVert B-O\rVert_F$, the orthogonal
  Procrustes solution: among all matrices with orthonormal columns, $UV^\top$ is closest to $B$.
  \item \emph{Steepest spectral direction.} Among updates of bounded operator norm,
  $\arg\max_{\lVert\Delta\rVert_{\mathrm{op}}\le 1}\langle B,\Delta\rangle =
  \operatorname{msign}(B)$: it is the direction that extracts the most first-order decrease per
  unit of spectral norm.
\end{enumerate}
\end{remark}

\paragraph{Computing it.} Rather than form an SVD, $\operatorname{msign}$ is approximated by a few
steps of a Newton--Schulz iteration on the spectrally normalised buffer,
\begin{equation}
  X \leftarrow a\,X + b\,(X X^\top)X + c\,(X X^\top)^2 X .
  \label{eq:ns}
\end{equation}
Every iterate shares $B$'s singular vectors, so \eqref{eq:ns} acts on each singular value
independently through the scalar map $\sigma \mapsto a\sigma + b\sigma^3 + c\sigma^5$. The fixed
odd-polynomial coefficients $(a,b,c)$ make $\sigma=1$ an attracting fixed point of that map, and
the initial normalisation places every singular value inside its basin; $k$ iterations therefore
flatten the whole spectrum towards $1$ using only matrix multiplies, which run stably on tensor
cores.

\paragraph{One frame for both optimisers.} The two updates are steepest descent under different
norms \citep{bernstein2024modular}:
\begin{equation}
  \arg\max_{\lVert\Delta\rVert_{\infty}\le\eta}\langle G,\Delta\rangle
  \;=\; \eta\,\operatorname{sign}(G),
  \qquad
  \arg\max_{\lVert\Delta\rVert_{\mathrm{op}}\le\eta}\langle G,\Delta\rangle
  \;=\; \eta\,\operatorname{msign}(G).
  \label{eq:twin}
\end{equation}
AdamW is approximately the left identity (sign descent under the elementwise max-norm, smoothed by
its moments), and Muon is exactly the right one. The spectral norm is the natural choice for a
weight matrix: for a layer $y=Wx$, an update $\Delta$ changes outputs by at most
$\lVert\Delta\rVert_{\mathrm{op}}\lVert x\rVert_2$, so bounding the operator norm bounds the
worst-case change in the layer's function. This frame also gives the relationship to Shampoo
\citep{gupta2018shampoo}, whose update is $P_L^{-1/4} G\, P_R^{-1/4}$ with left and right
preconditioners accumulating $GG^\top$ and $G^\top G$. In the single-gradient
(zero-accumulation) limit, $P_L=GG^\top=U\Sigma^2U^\top$ and $P_R=G^\top G=V\Sigma^2V^\top$, so
\begin{equation}
  P_L^{-1/4}\, G\, P_R^{-1/4}
  \;=\; U\Sigma^{-1/2}U^\top \cdot U\Sigma V^\top \cdot V\Sigma^{-1/2}V^\top
  \;=\; U\Sigma^{-1/2}\,\Sigma\,\Sigma^{-1/2}V^\top
  \;=\; UV^\top ,
  \label{eq:shampoo}
\end{equation}
identical to Muon's orthogonalised step (``Spectral Descent''). The optimisers agree on the
instantaneous geometry and differ in how second-moment information is accumulated over steps,
which is what the June-2026 ``is Muon just Shampoo?'' exchange concerned (\S\ref{sec:optimizers}). SOAP \citep{vyas2024soap}, which runs Adam inside Shampoo's eigenbasis,
sits between the two.

\paragraph{Implementation.} Megatron-Core's \texttt{TensorParallelMuon} applies Newton--Schulz
orthogonalisation on its orthogonalised-optimiser base \citep{nvidia2026megatronmuon}.
The ms-swift Megatron backend exposes the \texttt{muon} and
\texttt{dist\_muon} options and requires Megatron-Core 0.16 or newer
\citep{msswift2026megatron}. The exposed controls and their defaults are:

\begin{table}[htbp]
  \footnotesize
  \centering
  \caption{Muon controls in Megatron-Core / ms-swift and their defaults.}
  \label{tab:muon-impl}
  \begin{tabular}{lll}
    \toprule
    Parameter & Default & Role \\
    \midrule
    \texttt{optimizer} & \texttt{adam} & one of \texttt{adam, sgd, muon, dist\_muon} \\
    \texttt{muon\_momentum} & $0.9$ & momentum on the pre-orthogonalisation buffer \\
    \texttt{muon\_use\_nesterov} & False & Nesterov-style momentum \\
    \texttt{muon\_num\_ns\_steps} & $5$ & Newton--Schulz iterations \eqref{eq:ns} \\
    \texttt{muon\_scale\_mode} & \texttt{spectral} & update rescaling: \texttt{spectral / unit\_rms\_norm / shape\_scaling} \\
    \texttt{muon\_extra\_scale\_factor} & $1$ & extra multiplier on the update \\
    \texttt{muon\_fp32\_matmul\_prec} & \texttt{medium} & precision of the Newton--Schulz matmuls \\
    \texttt{muon\_tp\_mode} & \texttt{blockwise} & TP orthogonalisation: \texttt{blockwise / duplicated / distributed} \\
    \texttt{muon\_split\_qkv} & True & split a fused QKV projection before orthogonalising \\
    \bottomrule
  \end{tabular}
\end{table}

Three choices warrant comment. \emph{(i) Tensor-parallel orthogonalisation.}
$\operatorname{msign}$ is a global matrix operation, but a tensor-parallel weight is sharded
across ranks; \texttt{muon\_tp\_mode} decides how to reconcile the two. \texttt{blockwise} (the
default) orthogonalises each rank's shard independently (cheapest, approximate);
\texttt{duplicated} all-gathers the full matrix, orthogonalises it exactly, and re-shards;
\texttt{distributed} runs the Newton--Schulz iteration across ranks (exact, communication-heavy).
The mode trades orthogonalisation fidelity against communication, with no analogue in
AdamW. \emph{(ii) Fused-QKV splitting.} A fused QKV projection is one matrix in memory but three
logically distinct maps; with \texttt{muon\_split\_qkv} it is split so $W_Q,W_K,W_V$ are each
orthogonalised separately. \emph{(iii) The hybrid rule.} Orthogonalisation is meaningless for
vectors, so Muon is applied only to 2D hidden weights; one-dimensional parameters (norms, biases),
the embeddings, and the output head are optimised by Adam, the standard Muon practice. The
\texttt{spectral} scale mode additionally rescales the orthogonal update so its root-mean-square
magnitude matches an Adam step, letting a single learning rate transfer between the two arms. The
Kimi-K2 ``MuonClip'' variant, which clips attention logits for stability, is also available. Under
LoRA, Muon orthogonalises the 2D adapter matrices, so the optimised subspace matches the curvature
scope of Appendix~\ref{app:hessian}.

\begin{figure}[htbp]
  \centering
  \includegraphics[width=\linewidth]{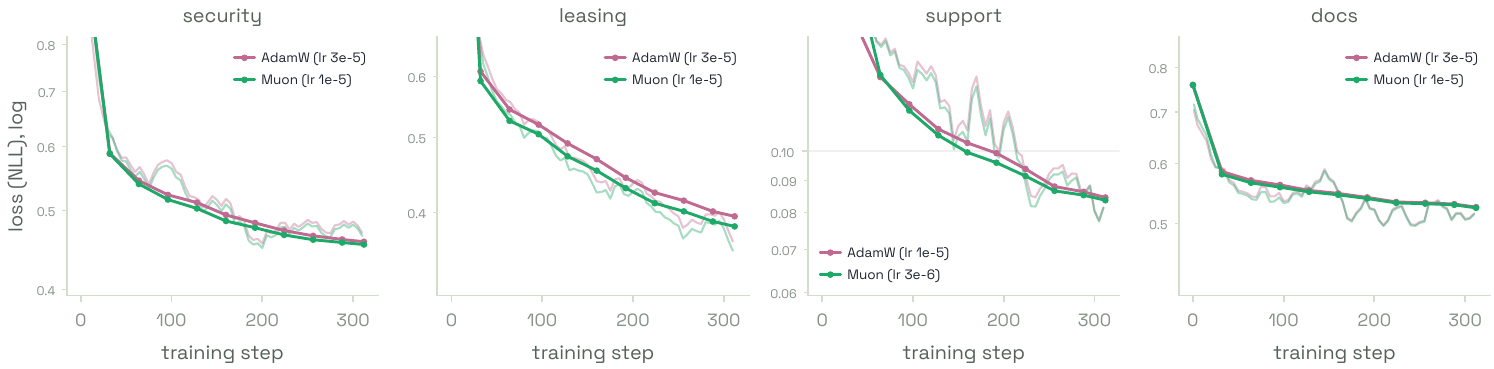}
  \caption{\textbf{Zoomed to the converged region, the validation trace separates: Muon ends at or just below AdamW on every dataset.} Validation NLL (markers, the selection metric) over training, with the smoothed per-step training loss faint behind, for each arm's best cell per dataset. The $y$-axis is logarithmic and clipped to the convergence band, so the steep early transient sits above the frame.}
  \label{fig:optim-curves}
\end{figure}

%% file: references.bib
@article{kaplan2020scaling,
  title   = {Scaling Laws for Neural Language Models},
  author  = {Kaplan, Jared and McCandlish, Sam and Henighan, Tom and others},
  journal = {arXiv preprint arXiv:2001.08361},
  year    = {2020}
}

@article{hoffmann2022chinchilla,
  title   = {Training Compute-Optimal Large Language Models},
  author  = {Hoffmann, Jordan and Borgeaud, Sebastian and Mensch, Arthur and others},
  journal = {arXiv preprint arXiv:2203.15556},
  year    = {2022}
}

@misc{thinkingmachines2025lora,
  title        = {{LoRA} Without Regret},
  author       = {{Thinking Machines Lab}},
  year         = {2025},
  howpublished = {Blog post, \url{https://thinkingmachines.ai/blog/lora/}},
  note         = {Accessed 2025}
}

@inproceedings{liu2023implicitbias,
  title     = {Same Pre-training Loss, Better Downstream: Implicit Bias Matters for Language Models},
  author    = {Liu, Hong and Xie, Sang Michael and Li, Zhiyuan and Ma, Tengyu},
  booktitle = {Proceedings of the 40th International Conference on Machine Learning (ICML)},
  series    = {PMLR},
  volume    = {202},
  year      = {2023}
}

@inproceedings{schaeffer2023mirage,
  title     = {Are Emergent Abilities of Large Language Models a Mirage?},
  author    = {Schaeffer, Rylan and Miranda, Brando and Koyejo, Sanmi},
  booktitle = {Advances in Neural Information Processing Systems (NeurIPS)},
  year      = {2023}
}

@article{gadre2024scaling,
  title   = {Language Models Scale Reliably with Over-training and on Downstream Tasks},
  author  = {Gadre, Samir Yitzhak and Smyrnis, Georgios and Shankar, Vaishaal and Wortsman, Mitchell and others},
  journal = {arXiv preprint arXiv:2403.08540},
  year    = {2024}
}

@article{owen2024predictable,
  title   = {How Predictable Is Language Model Benchmark Performance?},
  author  = {Owen, David},
  journal = {arXiv preprint arXiv:2401.04757},
  year    = {2024}
}

@inproceedings{saunshi2021mathematical,
  title     = {A Mathematical Exploration of Why Language Models Help Solve Downstream Tasks},
  author    = {Saunshi, Nikunj and Malladi, Sadhika and Arora, Sanjeev},
  booktitle = {International Conference on Learning Representations (ICLR)},
  year      = {2021}
}

@inproceedings{lourie2025unreliable,
  title     = {Scaling Laws Are Unreliable for Downstream Tasks: A Reality Check},
  author    = {Lourie, Nicholas and Hu, Michael Y. and Cho, Kyunghyun},
  booktitle = {Findings of the Association for Computational Linguistics: EMNLP 2025},
  year      = {2025}
}

@article{isik2024scaling,
  title   = {Scaling Laws for Downstream Task Performance of Large Language Models},
  author  = {Isik, Berivan and Ponomareva, Natalia and Hazimeh, Hussein and Paparas, Dimitris and
             Vassilvitskii, Sergei and Koyejo, Sanmi},
  journal = {arXiv preprint arXiv:2402.04177},
  year    = {2024}
}

@inproceedings{li2022sgdzeroloss,
  title     = {What Happens after {SGD} Reaches Zero Loss? A Mathematical Framework},
  author    = {Li, Zhiyuan and Wang, Tianhao and Arora, Sanjeev},
  booktitle = {International Conference on Learning Representations (ICLR)},
  year      = {2022}
}

@inproceedings{damian2021labelnoise,
  title     = {Label Noise {SGD} Provably Prefers Flat Global Minimizers},
  author    = {Damian, Alex and Ma, Tengyu and Lee, Jason D.},
  booktitle = {Advances in Neural Information Processing Systems (NeurIPS)},
  year      = {2021}
}

@inproceedings{smith2021origin,
  title     = {On the Origin of Implicit Regularization in Stochastic Gradient Descent},
  author    = {Smith, Samuel L. and Dherin, Benoit and Barrett, David G. T. and De, Soham},
  booktitle = {International Conference on Learning Representations (ICLR)},
  year      = {2021}
}

@article{bartlett1953,
  title   = {Approximate Confidence Intervals. {II}. More than One Unknown Parameter},
  author  = {Bartlett, M. S.},
  journal = {Biometrika},
  volume  = {40},
  number  = {3/4},
  pages   = {306--317},
  year    = {1953}
}

@article{martens2020,
  title   = {New Insights and Perspectives on the Natural Gradient Method},
  author  = {Martens, James},
  journal = {Journal of Machine Learning Research},
  volume  = {21},
  number  = {146},
  pages   = {1--76},
  year    = {2020}
}

@inproceedings{yao2020pyhessian,
  title     = {{PyHessian}: Neural Networks Through the Lens of the Hessian},
  author    = {Yao, Zhewei and Gholami, Amir and Keutzer, Kurt and Mahoney, Michael W.},
  booktitle = {IEEE International Conference on Big Data},
  year      = {2020}
}

@inproceedings{andriushchenko2023modern,
  title     = {A Modern Look at the Relationship between Sharpness and Generalization},
  author    = {Andriushchenko, Maksym and Croce, Francesco and M{\"u}ller, Maximilian and
               Hein, Matthias and Flammarion, Nicolas},
  booktitle = {Proceedings of the 40th International Conference on Machine Learning (ICML)},
  year      = {2023}
}

@article{granziol2020false,
  title   = {Flatness is a False Friend},
  author  = {Granziol, Diego},
  journal = {arXiv preprint arXiv:2006.09091},
  year    = {2020}
}

@inproceedings{dinh2017sharp,
  title     = {Sharp Minima Can Generalize For Deep Nets},
  author    = {Dinh, Laurent and Pascanu, Razvan and Bengio, Samy and Bengio, Yoshua},
  booktitle = {Proceedings of the 34th International Conference on Machine Learning (ICML)},
  year      = {2017}
}

@inproceedings{flatlora2025,
  title     = {Flat-{LoRA}: Low-Rank Adaptation over a Flat Loss Landscape},
  author    = {Li, Tao and others},
  booktitle = {Proceedings of the 42nd International Conference on Machine Learning (ICML)},
  year      = {2025}
}

@inproceedings{zhang2024finetuning,
  title     = {When Scaling Meets {LLM} Finetuning: The Effect of Data, Model and Finetuning Method},
  author    = {Zhang, Biao and Liu, Zhongtao and Cherry, Colin and Firat, Orhan},
  booktitle = {International Conference on Learning Representations (ICLR)},
  year      = {2024}
}

@article{kalajdzievski2023rslora,
  title   = {A Rank Stabilization Scaling Factor for Fine-Tuning with {LoRA}},
  author  = {Kalajdzievski, Damjan},
  journal = {arXiv preprint arXiv:2312.03732},
  year    = {2023}
}

@article{biderman2024lora,
  title   = {{LoRA} Learns Less and Forgets Less},
  author  = {Biderman, Dan and Portes, Jacob and Gonzalez Ortiz, Jose Javier and Paul, Mansheej and
             Greengard, Philip and Jennings, Connor and King, Daniel and Havens, Sam and
             Chiley, Vitaliy and Frankle, Jonathan and Blakeney, Cody and Cunningham, John P.},
  journal = {Transactions on Machine Learning Research (TMLR)},
  year    = {2024}
}

@inproceedings{muennighoff2023dataconstrained,
  title     = {Scaling Data-Constrained Language Models},
  author    = {Muennighoff, Niklas and Rush, Alexander M. and Barak, Boaz and Le Scao, Teven and
               Piktus, Aleksandra and Tazi, Nouamane and Pyysalo, Sampo and Wolf, Thomas and Raffel, Colin},
  booktitle = {Advances in Neural Information Processing Systems (NeurIPS)},
  year      = {2023}
}

@article{choshen2024hitchhiker,
  title   = {A Hitchhiker's Guide to Scaling Law Estimation},
  author  = {Choshen, Leshem and Zhang, Yang and Andreas, Jacob},
  journal = {arXiv preprint arXiv:2410.11840},
  year    = {2024}
}

@inproceedings{caballero2022broken,
  title     = {Broken Neural Scaling Laws},
  author    = {Caballero, Ethan and Gupta, Kshitij and Rish, Irina and Krueger, David},
  booktitle = {International Conference on Learning Representations (ICLR)},
  year      = {2023}
}

@article{besiroglu2024chinchilla,
  title   = {Chinchilla Scaling: A Replication Attempt},
  author  = {Besiroglu, Tamay and Erdil, Ege and Barnett, Matthew and You, Josh},
  journal = {arXiv preprint arXiv:2404.10102},
  year    = {2024}
}

@inproceedings{porian2024resolving,
  title     = {Resolving Discrepancies in Compute-Optimal Scaling of Language Models},
  author    = {Porian, Tomer and Wortsman, Mitchell and Jitsev, Jenia and Schmidt, Ludwig and Carmon, Yair},
  booktitle = {Advances in Neural Information Processing Systems (NeurIPS)},
  year      = {2024}
}

@article{bahri2024explaining,
  title   = {Explaining Neural Scaling Laws},
  author  = {Bahri, Yasaman and Dyer, Ethan and Kaplan, Jared and Lee, Jaehoon and Sharma, Utkarsh},
  journal = {Proceedings of the National Academy of Sciences (PNAS)},
  year    = {2024}
}

@inproceedings{loshchilov2019adamw,
  title     = {Decoupled Weight Decay Regularization},
  author    = {Loshchilov, Ilya and Hutter, Frank},
  booktitle = {International Conference on Learning Representations (ICLR)},
  year      = {2019}
}

@inproceedings{gupta2018shampoo,
  title     = {Shampoo: Preconditioned Stochastic Tensor Optimization},
  author    = {Gupta, Vineet and Koren, Tomer and Singer, Yoram},
  booktitle = {Proceedings of the 35th International Conference on Machine Learning (ICML)},
  year      = {2018}
}

@article{vyas2024soap,
  title   = {{SOAP}: Improving and Stabilizing Shampoo using Adam},
  author  = {Vyas, Nikhil and Morwani, Depen and Zhao, Rosie and Shapira, Itai and
             Brandfonbrener, David and Janson, Lucas and Kakade, Sham},
  journal = {arXiv preprint arXiv:2409.11321},
  year    = {2024}
}

@misc{jordan2024muon,
  title        = {Muon: An Optimizer for Hidden Layers in Neural Networks},
  author       = {Jordan, Keller and Jin, Yuchen and Boza, Vlado and You, Jiacheng and
                  Cesista, Franz and Newhouse, Laker and Bernstein, Jeremy},
  year         = {2024},
  howpublished = {Blog post, \url{https://kellerjordan.github.io/posts/muon/}}
}

@article{liu2025muonscalable,
  title   = {Muon is Scalable for {LLM} Training},
  author  = {Liu, Jingyuan and Su, Jianlin and others},
  journal = {arXiv preprint arXiv:2502.16982},
  year    = {2025}
}

@article{essential2025muon,
  title   = {Practical Efficiency of Muon for Pretraining},
  author  = {{Essential AI}},
  journal = {arXiv preprint arXiv:2505.02222},
  year    = {2025}
}

@article{bernstein2024modular,
  title   = {Old Optimizer, New Norm: An Anthology},
  author  = {Bernstein, Jeremy and Newhouse, Laker},
  journal = {arXiv preprint arXiv:2409.20325},
  year    = {2024}
}

@misc{jordan2026muonshampoo,
  title        = {Is Muon just Shampoo? Benchmark thread},
  author       = {Jordan, Keller},
  year         = {2026},
  howpublished = {Post on X, \url{https://x.com/kellerjordan0/status/2064062891607888058}},
  note         = {June 8, 2026}
}

@misc{anil2026muonshampoo,
  title        = {On the Muon--Shampoo comparison},
  author       = {Anil, Rohan},
  year         = {2026},
  howpublished = {Post on X, \url{https://x.com/_arohan_/status/2064216022232826080}},
  note         = {June 9, 2026}
}

@misc{nvidia2026megatronmuon,
  title        = {Megatron-Core Muon Optimizer},
  author       = {{NVIDIA}},
  year         = {2026},
  howpublished = {\url{https://github.com/NVIDIA/Megatron-LM/blob/dev/megatron/core/optimizer/muon.py}}
}

@misc{msswift2026megatron,
  title        = {Megatron-SWIFT: Command-line Parameters},
  author       = {{ModelScope}},
  year         = {2026},
  howpublished = {\url{https://github.com/modelscope/ms-swift}}
}

@article{goyal2017minibatch,
  title   = {Accurate, Large Minibatch {SGD}: Training {ImageNet} in 1 Hour},
  author  = {Goyal, Priya and Doll{\'a}r, Piotr and Girshick, Ross and Noordhuis, Pieter and
             Wesolowski, Lukasz and Kyrola, Aapo and Tulloch, Andrew and Jia, Yangqing and He, Kaiming},
  journal = {arXiv preprint arXiv:1706.02677},
  year    = {2017}
}

@article{mccandlish2018largebatch,
  title   = {An Empirical Model of Large-Batch Training},
  author  = {McCandlish, Sam and Kaplan, Jared and Amodei, Dario and {OpenAI Dota Team}},
  journal = {arXiv preprint arXiv:1812.06162},
  year    = {2018}
}

@article{shallue2019dataparallelism,
  title   = {Measuring the Effects of Data Parallelism on Neural Network Training},
  author  = {Shallue, Christopher J. and Lee, Jaehoon and Antognini, Joseph and Sohl-Dickstein, Jascha and
             Frostig, Roy and Dahl, George E.},
  journal = {Journal of Machine Learning Research},
  volume  = {20},
  number  = {112},
  pages   = {1--49},
  year    = {2019}
}

@inproceedings{yang2021mup,
  title     = {Tuning Large Neural Networks via Zero-Shot Hyperparameter Transfer},
  author    = {Yang, Greg and Hu, Edward J. and Babuschkin, Igor and Sidor, Szymon and Liu, Xiaodong and
               Farhi, David and Ryder, Nick and Pachocki, Jakub and Chen, Weizhu and Gao, Jianfeng},
  booktitle = {Advances in Neural Information Processing Systems (NeurIPS)},
  year      = {2021}
}

@inproceedings{smith2018increasebatch,
  title     = {Don't Decay the Learning Rate, Increase the Batch Size},
  author    = {Smith, Samuel L. and Kindermans, Pieter-Jan and Ying, Chris and Le, Quoc V.},
  booktitle = {International Conference on Learning Representations (ICLR)},
  year      = {2018}
}

@article{du2024emergentloss,
  title   = {Understanding Emergent Abilities of Language Models from the Loss Perspective},
  author  = {Du, Zhengxiao and Zeng, Aohan and Dong, Yuxiao and Tang, Jie},
  journal = {arXiv preprint arXiv:2403.15796},
  year    = {2024}
}

@article{zhou2023ifeval,
  title={Instruction-Following Evaluation for Large Language Models},
  author={Zhou, Jeffrey and Lu, Tianjian and Mishra, Swaroop and Brahma, Siddhartha and Basu, Sujoy and Luan, Yi and Zhou, Denny and Hou, Le},
  journal={arXiv preprint arXiv:2311.07911},
  year={2023}
}

@misc{oneill2025isft,
  title        = {Iterative {SFT} ({iSFT}): Dense Reward Learning},
  author       = {O'Neill, Charles and Liu, Jonathon and Partridge, Harry and Kirkby, Max and Jayasekara, Mudith},
  year         = {2025},
  howpublished = {Baseten Research},
  note         = {\url{https://www.baseten.co/research/iterative-sft/}}
}
